\documentclass{ifacconf}

\usepackage{graphicx}      
\usepackage{natbib}        

\usepackage{booktabs}       
\usepackage{amsfonts}       
\usepackage{nicefrac}       
\usepackage{microtype}      

\usepackage{amsmath}
\usepackage{amssymb}
\usepackage{xcolor}
\usepackage{algorithm}
\usepackage{algcompatible}
\usepackage{balance}
\usepackage{multirow} 
\usepackage{threeparttable}
\usepackage{bbding}

\def \hy #1{\textcolor{black}{#1}}

\newcommand{{\myalg}}{ST-GT}

\newenvironment{proof}{\noindent\textit{Proof.}\ }{\hfill$\square$\par}

\makeatletter
\def\title@fmt#1#2{%
  \@ifundefined{@runtitle}{\global\def\@runtitle{#1}}{}%
  \@articletypesize
  \leavevmode\vphantom{Aye!}%
  \@articletype
  \vskip12\p@
  \hbox to\textwidth{%
    \hss
    \parbox{1.10\textwidth}{
      \centering
      {\@titlesize #1\,\hbox{$^{#2}$}\par}%
    }%
    \hss
  }%
  \vskip\@undertitleskip
}
\makeatother

\begin{document}

\begin{frontmatter}

\title{Beyond Scaffold: A Unified Spatio-Temporal \\ Gradient Tracking Method\thanksref{footnoteinfo}}

\thanks[footnoteinfo]{The computations/data handling/[SIMILAR] were/was enabled by resources provided by the National Academic Infrastructure for Supercomputing in Sweden (NAISS), partially funded by the Swedish Research Council through grant agreement no. 2022-06725.}

\author[First]{Yan Huang} 
\author[Second]{Jinming Xu} 
\author[Second]{Jiming Chen} 
\author[First]{Karl Henrik Johansson}

 \address[First]{Division of Decision
and Control Systems, School of EECS, KTH Royal Institute of Technology,
SE-100 44 Stockholm, Sweden (e-mail: yahuang@kth.se, kallej@kth.se)}

\address[Second]{College of Control Science and Engineering, Zhejiang University, 310027 Hangzhou, China (e-mail: jimmyxu@zju.edu.cn, cjm@zju.edu.cn)}

\begin{abstract}                
In distributed and federated learning algorithms, communication overhead is often reduced by performing multiple local updates between communication rounds. However, due to data heterogeneity across nodes and the local gradient noise within each node, this strategy can lead to the drift of local models away from the global optimum. To address this issue, we revisit the well-known federated learning method Scaffold \citep{karimireddy2020scaffold} under a gradient tracking perspective, and propose a unified spatio-temporal gradient tracking algorithm, termed {\myalg}, for distributed stochastic optimization over time-varying graphs. {\myalg} tracks the global gradient across neighboring nodes to mitigate data heterogeneity, while maintaining a running average of local gradients to substantially suppress noise, with slightly more storage overhead. Without assuming bounded data heterogeneity, we prove that {\myalg} attains a linear convergence rate for strongly convex problems and a sublinear rate for nonconvex cases. Notably, {\myalg} achieves the first linear speed-up in communication complexity with respect to the number of local updates per round $\tau$ for the strongly-convex setting. Compared to traditional gradient tracking methods, {\myalg} reduces the topology-dependent noise term from $\sigma^2$ to $\sigma^2/\tau$, where $\sigma^2$ denotes the noise level, thereby improving communication efficiency. 
\end{abstract}


\begin{keyword}
 Distributed optimization, federated learning, data heterogeneity.
\end{keyword}

\end{frontmatter}

\section{Introduction}
\label{sec:intro}

Data parallelism is a standard paradigm in large-scale machine learning tasks, in which multiple devices collaboratively train a shared model \citep{dean2012large}. Common strategies include employing a centralized server–worker architecture, as in federated learning \citep{Li2014scaling}, where a server aggregates gradients and distributes model updates. Alternatively, a distributed network architecture allows each node to compute locally and exchange information with its neighbors \citep{yuan2016convergence}, making it well-suited for training over multiple data centers \citep{lian2017can}, wireless sensor networks \citep{rabbat2004distributed}, and multi-robot systems \citep{tian2022distributed}. In general, both paradigms aim to solve the following distributed optimization problem with $n$ nodes:
\begin{equation}\label{Def_prob} 
\underset{x\in \mathbb{R} ^p}{\min}f\left( x \right) =\frac{1}{n}\sum_{i=1}^n{\underset{:=f_i\left( x \right)}{\underbrace{\mathbb{E} _{\xi _i\sim \mathcal{D} _i}\left[ f_i\left( x;\xi _i \right) \right] }},}
\end{equation}
where $x\in \mathbb{R} ^p$ is the global model parameter, $f_i$ the local objective function, and $\xi_i$ the random sample drawn from the local data distribution $\mathcal{D} _i$ accessible only to node $i$. 

To reduce communication overhead between neighboring nodes or between nodes and the central server, it is common to skip certain communication rounds and perform multiple local updates in federated and distributed learning \citep{mcmahan2017communication, nguyen2023performance}. However, due to data heterogeneity across nodes and the sampling noise inherent in local gradients, reducing communication can substantially increase the drift of local models from the global gradient direction \citep{karimireddy2020scaffold} and amplify gradient noise \citep{huang2023computation}, resulting in additional errors or requiring diminishing stepsizes. Addressing these issues while preserving communication efficiency remains an active research problem. Moreover, although distributed learning and federated learning share nearly identical objectives and challenges, existing studies often treat them separately, which hinders a unified understanding of data-parallel learning algorithms.

\subsection{Related Work}
\label{sec:prior}

\textbf{Federated learning.} 
Canonical federated learning methods adopt a server–worker architecture, where sampled worker nodes perform local computations in a data-parallel manner and then communicate with the server node to achieve information aggregation. A major issue with this approach is that the server node becomes a communication bottleneck and a potential single point of failure \citep{zhang2023fedaudio}. 
To reduce the communication load per round on the server, \cite{mcmahan2017communication} proposed the FedAvg algorithm, which improves performance by allowing only a subset of nodes to participate in training and by incorporating multiple rounds of local updates. FedAvg has since been widely applied in many real-world scenarios, particularly for the case of identically and independently distributed (i.i.d.) data \citep{stich2018local}. However, \cite{karimireddy2020scaffold} demonstrated that FedAvg fails to guarantee exact convergence under heterogeneous data distributions. The reason is that each worker may converge toward its own local optimum based solely on its dataset, resulting in client drift. To mitigate this issue, they proposed the Scaffold algorithm, which introduces control variables at both the server and worker sides to correct local gradient directions, thereby counteracting client drift. 
This method was later extended to various settings, including those with random communication intervals \citep{mishchenko2022proxskip}, finite-sum problems with variance reduction \citep{jiang2024federated}, and federated compositional optimization problems \citep{zhang2024composite}, to name a few, achieving improved performance in non-i.i.d. scenarios.
Another line of research for addressing data heterogeneity is personalized federated learning, which includes approaches such as regularization-based methods \citep{li2020federated} and knowledge distillation \citep{lee2022preservation}. For a comprehensive overview, we refer readers to recent surveys on this topic \citep{mora2024enhancing}.

\textbf{Distributed learning.} 
In a distributed network without a central server, each node communicates only with its immediate neighbors \citep{ram2009asynchronous}. Such a topology is more flexible and robust against single-point failures. 
However, relying solely on peer-to-peer communication also makes distributed learning algorithms vulnerable to data heterogeneity, and this effect becomes increasingly pronounced as network connectivity weakens \citep{lian2017can}. To address this issue, gradient tracking (GT) algorithms have become a mainstream choice \citep{xu2015augmented, di2016next, nedic2017achieving}. For instance, \cite{xu2015augmented} proposed Aug-DGM, which incorporates a global gradient estimator with a dynamic average-consensus protocol. This method progressively eliminates the effect of data heterogeneity and achieves linear convergence with a constant stepsize for strongly-convex and smooth objectives. Aug-DGM was later extended to stochastic gradient settings \citep{pu2021distributed} and further applied to various machine learning tasks \citep{lu2021optimal}. 
To improve the communication efficiency, \cite{nguyen2023performance} integrated GT with multiple local updates and proposed the LU-GT algorithm for deterministic settings. \cite{wu2025effectiveness} showed that LU-GT remains communication-efficient under mild data heterogeneity. 
However, simply skipping communication will, indeed, amplify the impact of gradient noise, thereby increasing computational complexity.
\cite{huang2023computation} analyzed the trade-off between communication and computation costs with respect to the number of local updates, and proposed FlexGT, a flexible gradient tracking method for stochastic settings that supports adjustable computation and communication steps.
Subsequently, \cite{liu2024decentralized} introduced the K-GT algorithm, which employs a Scaffold-style control variable to correct local gradients and reduce gradient noise by scaling with the number of local updates; however, this improvement hinges on initializing the control variables with global gradients and exhibits a stronger dependence on the network topology.

\subsection{Contributions} 
In this work, we solve the distributed stochastic optimization problem \eqref{Def_prob} with non-independent and identically distributed (non-i.i.d.) datasets. The main contributions are summarized as follows: 
\begin{itemize}
     \item We propose a unified spatio-temporal gradient tracking algorithm ({\myalg}) for time-varying graphs, applicable to both distributed and federated learning. By jointly tracking the global gradient across nodes and the running average of local gradients, {\myalg} effectively mitigates data heterogeneity and reduces gradient noise, incurring only a slight increase in storage overhead. Furthermore, by extending Scaffold with a tunable parameter in its global control variable, we show that Scaffold emerges as a special case of {\myalg} under a random communication topology, thereby offering a unified framework that bridges distributed and federated learning paradigms.
    \item Without assuming any bound on data heterogeneity, we prove that {\myalg} achieves a linear convergence rate (Theorem~\ref{Thm_sc}) and a sublinear rate (Theorem~\ref{Thm_nc}) for strongly convex and nonconvex objective functions, respectively. More importantly, {\myalg} achieves the first linear speed-up in communication complexity with respect to the number of local updates per round $\tau$ in the strongly convex setting, improving upon FlexGT by a factor of $1/\tau$. It also reduces the network dependence from $1/(1-\rho)^2$ to $1/(1-\rho)^{3/2}$ compared with K-GT (see Table~\ref{Tab_comparison_of_related_work}), where $\rho$ denotes the graph connectivity. The theoretical results are validated on both synthetic and real-world datasets.
\end{itemize}

\textbf{Paper organization.}
The remainder of the paper is organized as follows. Section~\ref{sec:formulation} formulates the distributed stochastic optimization problem and introduces the design of {\myalg}. Section~\ref{sec:recovering} studies an extended Scaffold algorithm and establishes its connections with {\myalg}. Section~\ref{sec:conv_results} presents the convergence results of {\myalg}. Section~\ref{sec:simulation} reports numerical experiments that validate the theoretical analysis. Finally, Section~\ref{sec:conclusion} concludes the paper, and Appendix~\ref{sec:appendix} contains several supporting lemmas and the proofs of the main results.

\textbf{Notations.} 
In this work, we use the following notation: $\|\cdot\|$ denotes the Frobenius norm, $\langle \cdot , \cdot \rangle$ the inner product, $\left| \cdot \right|$ the cardinality of a set, and $\mathbb{E}[\cdot]$ the expectation of a vector or matrix. We let $\mathbf{1}$ be the all-ones vector, $\mathbf{I}$ the identity matrix, and define the averaging matrix as $\mathbf{J} = \mathbf{1}\mathbf{1}^{\top}/n$. In addition, the asymptotic notation $\mathcal{O}(\cdot)$ is used to suppress constant factors, while $\tilde{\mathcal{O}}(\cdot)$ further omits logarithmic factors.

\section{Problem Formulation and {\myalg} Algorithm}
\label{sec:formulation}

\subsection{Distributed Stochastic Optimization}

For the implementation purpose, we consider the following equivalent problem with consensus constraints:

\begin{equation}\label{Def_prob_con} 
\begin{aligned}
\underset{X\in \mathbb{R} ^{n\times p}}{\min} &F\left( X \right) =\frac{1}{n}\sum_{i=1}^n{\underset{:=f_i\left( x_i \right)}{\underbrace{\mathbb{E} _{\xi _i\sim \mathcal{D} _i}\left[ f_i\left( x_i;\xi _i \right) \right] }}},
\\
&\text{s.t.} \,\,x_i=x_j, \,\, i,j\in[n],
\end{aligned}
\end{equation}
where $X=\left[ x_1,\dots ,x_n \right] ^{\top}\in \mathbb{R} ^{n\times p}$ is the collection of the local decision variable $x_i\in \mathbb{R} ^{p}$ of each node $i$. The nodes communicate over a network whose topology is represented by a graph $\mathcal{G}=(\mathcal{V},\mathcal{E})$, where $\mathcal{V}={1,2,\ldots,n}$ is the set of agents and $\mathcal{E}\subseteq\mathcal{V}\times \mathcal{V}$ is the set of edges, with each edge $(i,j)$ indicating a communication link between agents $i$ and $j$. Each agent $i$ exchanges information only with its neighbors, defined as $\mathcal{N}_i={ j \mid j\ne i,, (i,j)\in \mathcal{E} }$, together with itself. 
Particularly, to accommodate both gossip-based communication in distributed learning and partial client participation in federated learning, we consider a general dynamic graph $\mathcal{G}_r$ at each communication round $r$, satisfying the following assumption. 
\begin{assum}[Connectivity in expectation] \label{Ass_graph}
The weight matrix $W_r$ induced by a dynamic graph $\mathcal{G}_r$ is doubly stochastic, i.e., $W_r\mathbf{1}=\mathbf{1},\mathbf{1}^{\top}W_r=\mathbf{1}^{\top}$ and $\rho :=\mathbb{E} \left[ \left\| W_r-\mathbf{J} \right\| _{2}^{2} \right] <1,\forall r\geqslant 0$.
\end{assum}
Note that this assumption requires the underlying communication graph to satisfy a contraction property only in expectation, rather than at every round as assumed in \cite{nguyen2023performance, liu2024decentralized}. \hy{This relaxation allows for more flexible and communication-efficient network topologies \citep{ying2021exponential, nguyen2025graphs}. }

\subsection{The {\myalg} Algorithm}

In this work, we address Problem \eqref{Def_prob_con} with non-i.i.d. local datasets, i.e., $\mathcal{D} _i\ne \mathcal{D} _j, i\ne j$. 
For brevity, we denote
\begin{equation*}
\begin{aligned}
X_k&:=\left[ \dots ,x_{i,k},\dots \right] ^{\top}, Y_k:=\left[ \dots ,y_{i,k},\dots \right] ^{\top}\in \mathbb{R} ^{n\times p},
\\
G_k&:=\left[ \dots ,\nabla f_i\left( x_{i,k};\xi _{i,k} \right) ,\dots \right] ^{\top}\in \mathbb{R} ^{n\times p},
\\
\nabla F_k&:=\left[ \dots ,\nabla f_i\left( x_{i,k} \right) ,\dots \right] ^{\top}\in \mathbb{R} ^{n\times p}.
\end{aligned}
\end{equation*}
as the collections of the local model parameters, gradient tracking variables, stochastic gradient, and full gradient, respectively, at iteration $k$. 

Recall the GT method with local updates \citep{nguyen2023performance, huang2023computation, wu2025effectiveness}, which follows the update rules given below:
\begin{equation}
\begin{aligned}
X_{\left( r+1 \right) \tau}&=W\left( X_{r\tau}-\gamma \sum_{t=0}^{\tau -1}{Y_{r\tau +t}} \right) ,
\\
Y_{\left( r+1 \right) \tau}&=WY_{r\tau}+G_{\left( r+1 \right) \tau}-G_{r\tau},
\end{aligned}
\end{equation}
where $W$ indicates a fixed weight matrix induced by the graph, and $\tau$ is the number of local updates between communications. \hy{It can be observed that the update of the model parameter $X$ relies on the accumulated tracking variables within the $r$-th round, whereas $Y$ uses only the single-step values without aligning with the accumulated quantity in the updates of $X$. This mismatch is inconsistent with the principle of gradient tracking.}

Motivated by these observations, we propose the spatio-temporal gradient tracking algorithm, {\myalg}, whose pseudo-code is given in Algorithm~\ref{Alg_ST_GT}. 
The key idea is illustrated in Fig.~\ref{fig:STGT}.
By caching the model parameters from the previous communication round and passing them to the next communication round for difference calculation (the link from iteration $r\tau$ to $r\tau+\tau$), the algorithm accumulates the temporal tracking variable $Z$, which is then mixed spatially via the weighting matrix $W_r$. This spatio-temporal gradient tracking mechanism effectively approximates centralized gradient descent with only a slight increase in storage overhead, improving convergence performance while preserving communication efficiency.
In particular,
{\myalg} can be rewritten in the following compact form: for $k\in[r\tau,\,\, r\tau+\tau-2]$,
\begin{equation}
\begin{aligned}
X_{k+1}&=X_{k}-\gamma Y_{k},
\\
Y_{k+1}&=Y_{k}+G_{k+1}-G_{k};
\end{aligned}
\end{equation}
and for the $ k = (r+1)\tau$,
\begin{equation}\label{Eq_ST_GT}
\begin{aligned}
Z_{r \tau}&=\frac{1}{\gamma \tau}\left( X_{r\tau}-X_{r\tau +\tau -1} +\gamma Y_{r\tau +\tau -1} \right),
\\
X_{\left( r+1 \right) \tau}&=W_r\left( X_{r\tau}-\tau \gamma Z_{r \tau} \right) ,
\\
Y_{\left( r+1 \right) \tau}&=W_r Z_{r \tau}+G_{\left( r+1 \right) \tau}-\frac{1}{\tau}\sum_{t=0}^{\tau -1}{G_{r\tau +t}.}
\end{aligned}
\end{equation}
Intuitively, by applying the dynamic consensus protocol with the doubly stochastic matrix $W_r$,  {\myalg} asymptotically tracks the running average of the global gradient within each period, i.e.,
\begin{equation}
\begin{aligned}
\bar{z}_{\left( r+1 \right) \tau}&:=\frac{1}{n}\sum_{i=1}^n{z_{i,\left( r+1 \right) \tau}}
\\
&=\frac{1}{n}\sum_{i=1}^n{\frac{1}{\tau}\sum_{t=0}^{\tau -1}{y_{i,r\tau +t}}}=\frac{1}{n}\sum_{i=1}^n{\frac{1}{\tau}\sum_{t=0}^{\tau -1}{g_{i,r\tau +t}}}.
\end{aligned}
\end{equation}
\hy{This enhances robustness to the gradient noise with only a slight increase in memory overhead, compared to GT methods with variance reduction (GT-VR) \citep{xin2020variance}, which typically incur higher memory or computational costs to approximate the full gradient (see Table \ref{Tab_comparison_of_related_work}).
Instead, DSGT \citep{pu2021distributed} and FlexGT \citep{huang2023computation} track the global gradient at a single step, i.e.,
\begin{equation}
\bar{y}_{\left( r+1 \right) \tau}:=\frac{1}{n}\sum_{i=1}^n{y_{i,\left( r+1 \right) \tau}}=\frac{1}{n}\sum_{i=1}^n{g_{i,\left( r+1 \right) \tau}}.
\end{equation}
Particularly, {\myalg} reduces to DSGT when $\tau=1$, and to FlexGT when communicating $Y$ instead of $Z$. Moreover, unlike K-GT \citep{liu2024decentralized}, ST-GT does not require any extra communication during the initialization phase.}

\begin{figure}[t]
\begin{minipage}[b]{0.95\linewidth}
  \centering
  \centerline{\includegraphics[width=0.9\linewidth]{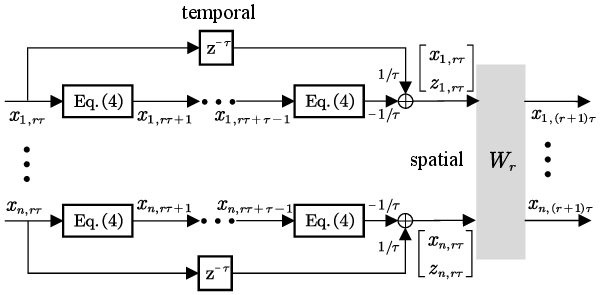}}
\end{minipage}
\caption{Illustration of the spatio-temporal gradient tracking. \hy{The link from time $r\tau$ to $r\tau+\tau$, passing through a memory element $\text{z}^{-\tau}$, illustrates how ST-GT tracks local gradient along the temporal dimension.}}
\label{fig:STGT}
\end{figure}

\begin{algorithm}[!htpb]
\caption{\textbf{{\myalg} (distributed)}}
\label{Alg_ST_GT}
\begin{algorithmic}[1]
\renewcommand{\algorithmicrequire}{\textbf{Initialization: }}
\REQUIRE {Initial points $x_{i,0}\in \mathbb{R}^p$ and $\tilde{g}_{i,0}=y_{i,0}=\nabla_x f_i\left( x_{i,0};\xi _{i,0} \right)$, number of local updates $\tau \geqslant 1$ and stepsize $\gamma >0$.
}
\
\FOR{round $r = 0,1,\cdots, R-1$, each node $i\in[n]$,}

\STATE {Re-initial the running-average gradient $\tilde{g}_{i,r\tau}=0$}

\FOR{$k=r\tau,r\tau+1,\cdots,(r+1)\tau-2$}
\STATE {Sample stochastic gradient $g_{i,r\tau+t}$.}

\STATE Perform local update:
\vspace{-0.2cm}
\begin{equation*}
\begin{aligned}
x_{i,k+1}&=x_{i,k}-\gamma y_{i,k},
\\
y_{i,k+1}&=y_{i,k}+{g}_{i,k+1}-{g}_{i,k},
\\
\tilde{g}_{i,k+1}&=\tilde{g}_{i,k}+g_{i,k+1}.
\end{aligned}
\end{equation*}
\ENDFOR

\STATE $z_{i,\left( r+1 \right)\tau}=\frac{1}{\gamma \tau}\left( x_{i,r\tau}-x_{i,r\tau+\tau-1}+\gamma y_{i,r\tau+\tau-1} \right) $.

\STATE Perform inter-node communication:

\vspace{-0.2cm}
\begin{equation*}
\begin{aligned}
x_{i,\left( r+1 \right)\tau}&=\sum_{j\in \mathcal{N} _i}{w^r_{i,j}\left( x_{j,r\tau}-\tau \gamma z_{j,\left( r+1 \right)\tau} \right)},
\\
y_{i,\left( r+1 \right)\tau}&=\sum_{j\in \mathcal{N} _i}{w^r_{i,j}z_{j,\left( r+1 \right)\tau}}+g_{i,\left( r+1 \right)\tau}-\frac{1}{\tau}\tilde{g}_{i,(r+1)\tau}.
\end{aligned}
\end{equation*}
\ENDFOR
\end{algorithmic}
\end{algorithm}

\section{Connection to Scaffold}\label{sec:recovering}

\begin{table*}
    \footnotesize 
    \begin{center}
        \caption{Relevant algorithms for solving Problem \eqref{Def_prob_con} with strongly-convex (SC) and nonconvex (NC) objective functions. We take the size of the model parameters as one unit and compare the overhead of related distributed optimization algorithms in terms of per-node computational (Comp.), communication (Comm.), and memory costs, as well as their communication complexity to achieve an accuracy of $\epsilon>0$. Here, $m$ denotes the total number of local samples at each node. “s/w” denotes the server–worker architecture, and `dist.' the distributed graph. }
        \label{Tab_comparison_of_related_work}
        {
        \begin{threeparttable}
            \begin{tabular}{c|c|c|c|c|c|c}
            \hline
                \rule{0pt}{10pt}
                {\textbf{Algorithm}}
                &{\textbf{Graph}}
                & {\textbf{Comp.}}
                & {\textbf{Comm.}}
                &{\textbf{Memory}}
                & {\textbf{Complexity} ($\tilde{\mathcal{O}}(\cdot)$ )}
                & {\textbf{Assum.}}
                \\
                \hline
                \rule{0pt}{15pt}
                \multirow{2}{*}{Scaffold \citep{karimireddy2020scaffold}} \tnote{a}&  \multirow{2}{*}{s/w}  & \multirow{2}{*}{1x} & \multirow{2}{*}{2x} & \multirow{2}{*}{6x} & $\frac{L}{\mu}+\frac{n}{s}+\frac{\sigma ^2}{\mu ^2\tau n\epsilon}\frac{n}{s}$ & SC
                \rule{0pt}{15pt}
                \\
                \rule{0pt}{15pt}
                 &   & &  &  & $\frac{\sigma ^2}{\tau \epsilon ^2n}\frac{n}{s}+\frac{1}{\epsilon}\left( \frac{n}{s} \right) ^{2/3}$ & NC
                 \rule{0pt}{15pt}
                \\
                 \hline
                \rule{0pt}{15pt}
                 FlexGT \citep{huang2023computation} \tnote{b} &  dist.  & $1$x & 2x &  4x & $\frac{L}{\left( 1-\rho \right) ^2\mu}+\frac{\sigma ^2/\tau}{\mu ^2n\epsilon}+\frac{\sqrt{L\sigma ^2}}{\sqrt{\mu ^3\left( 1-\rho \right) ^3\epsilon}}$ & SC
                 \rule{0pt}{15pt} 
                 \\
                \hline
                \rule{0pt}{15pt}
                 K-GT \citep{liu2024decentralized} &  dist. & 1x & 2x & 5x & $\frac{L}{\left( 1-\rho \right) ^2\epsilon}+\frac{L\sigma ^2}{\tau n\epsilon ^2}+\frac{L\sigma}{\left( 1-\rho \right) ^2\sqrt{\tau \epsilon ^3}}$ & NC
                 \rule{0pt}{15pt}
                \\
                \hline
                \rule{0pt}{15pt}
                GT-VR \citep{xin2020variance} \tnote{c} &  dist.  & \{1, m\}x & 2x & \{ 4+m, 4\}x & $\max \left\{ m, \frac{L}{\mu \left( 1-\rho \right) ^2} \right\} \log \frac{1}{\epsilon}$ & SC
                \rule{0pt}{15pt} 
                \\
                \hline
                \rule{0pt}{15pt}
                \multirow{2}{*}{{\myalg} \text{(This work)}}&  \multirow{2}{*}{dist.}  & \multirow{2}{*}{1x} & \multirow{2}{*}{2x} & \multirow{2}{*}{5x} & $\frac{L}{\left( 1-\rho \right) ^2\mu}+\frac{\sigma ^2}{\mu ^2n\tau \epsilon}+\sqrt{\frac{L\sigma ^2}{\mu ^3\left( 1-\rho \right) ^3\tau \epsilon}}$ & SC
                \rule{0pt}{15pt}
                \\
                \rule{0pt}{15pt}
                 &   & &  &  & $\frac{L}{\left( 1-\rho \right) ^2\epsilon}+\frac{\sigma ^2L}{n\tau \epsilon ^2}+\frac{L\sqrt{\sigma ^2/\tau}}{\sqrt{\left( 1-\rho \right) ^3\epsilon ^3}}+\frac{C_0}{\epsilon}$ & NC
                 \rule{0pt}{15pt}
                \\
                \hline
                
            \end{tabular}
            
            \begin{tablenotes}
                \item[a] \hy{Note that $s/n$ can be interpreted as a measure of connectivity in the server-worker topology. However, it is not directly comparable to $1-\rho$, as obtaining a closed-form expression for $\mathbb{E} \left[ \left\| W_r-\mathbf{J} \right\| ^2 \right]$ with $W_r$ defined in \eqref{Eq_Wr} is generally intractable.}
               \item[b] This rate corresponds to the case where no multi-round communication is employed, consistent with the setting considered in this paper.
               \item[c] This rate incurs higher memory or computational costs to approximate the full gradient and is obtained under the assumption of sample-wise smoothness of the objective function.
            \end{tablenotes}
        \end{threeparttable}
        }
    \end{center}
\end{table*}

In this section, we investigate an extended version of the Scaffold algorithm \citep{karimireddy2020scaffold} and build its connection with {\myalg} in the context of more general distributed and dynamic network settings.
\subsection{Extended Scaffold with \hy{Parameter $\gamma_c$}}

\begin{algorithm}[t]
\caption{\textbf{Scaffold$^+$ (server--worker)}}
\label{alg:scaffold}
\begin{algorithmic}[1]
\STATE Initialization: server and worker model parameters $x_0, x_{i,0}$, global and local control variables $c_0, c_{i,0}$, stepsizes $\gamma_g$, $\gamma_l$ and $\gamma_c$.

\FOR{$r=0,\dots,R-1$}

\STATE Sample worker nodes $S_r\subseteq \left[ 1,\dots, n \right] $.
\STATE {Send $x_{r\tau}$ and $c_{r\tau}$ to sampled nodes.}

\FOR{$i\in S_r$}

\STATE {initial the local mode $x_{i,r\tau}=x_{r\tau}$.}

\FOR{$k=r\tau,r\tau+1,\cdots,r\tau+\tau-2$}
\STATE {Sample stochastic gradient $g_{i,k}$.}
\STATE Local updates
\begin{equation*}
\begin{aligned}
x_{i,k+1}&=x_{i,k}-\gamma_l \left( c_k+g_{i,k} -c_{i,k} \right) ,
\\
x_{k+1}&=x_k, \,\,c_{i,k+1}=c_{i,k},\,\, c_{k+1}=c_k.
\end{aligned}
\end{equation*}
\ENDFOR

\STATE {For $k=r\tau+\tau-1$}
\begin{equation*}
\begin{aligned}
&x_{i,k+1/2}=x_{i,k}-\gamma _l\left( c_k+g_{i,k} -c_{i,k} \right) ,
\\
&c_{i,k+1}=c_{i,k}-c_k+\frac{1}{\tau\gamma_l}\left( x_k-x_{i,k+1/2} \right).
\end{aligned}
\end{equation*}
\ENDFOR

\STATE {Communicate to server and update}
\begin{equation*}
\begin{aligned}
&x_{k+1}=x_k+\frac{\gamma_g}{\left| S_r \right|}\sum_{j\in S_r}{\left( x_{j,k+1/2}-x_k \right) ,}
\\
&c_{k+1}=c_k+\frac{{\gamma_c}}{\left| S_r \right|}\sum_{j\in S_r}{\,\,\left( c_{j,k+1}-c_{j,k} \right)}.
\end{aligned}
\end{equation*}

\ENDFOR
\end{algorithmic}
\end{algorithm}

\hy{Scaffold is a popular federated learning algorithm designed to address data heterogeneity in non-i.i.d. settings \citep{karimireddy2020scaffold}.
It employs local and global control variables, $c_i$ for each worker node and $c$ for the server, to correct the client drift caused by data heterogeneity. Building on Scaffold, we propose an extended version in Algorithm~\ref{alg:scaffold}, named Scaffold$^+$, which introduces a tunable parameter $\gamma_c \geqslant 0$ (cf., line 13) in that algorithm. 
This modification generalizes the original Scaffold (recovered when $\gamma_c=\left| S_r \right|/n$, where $S_r\subseteq \left[ 1,\dots, n \right]$ denotes the set of sampled workers at each round $r$) and establishes a connection to {\myalg}.}
In particular, define
\[y_{i,k}:=c_k+g_{i,k} -c_{i,k},\]
which serves as the gradient tracking variable in {\myalg}. 
We show that Scaffold$^+$ can be interpreted as a spatiao-temporal gradient tracking method with dynamic networks in the following proposition.

\begin{prop}
\label{Prop_reform_scaffold}
Consider Scaffold$^+$ in Algorithm~\ref{alg:scaffold}. 
For each node $i \in S_r$ at the $r$-th communication round, we have $x_{i,r\tau}=x_{r\tau}$ and  $c_{i,r\tau} = c_{i,t(i)}$, where $t(i) \le r\tau$ denotes the most recent round at which node $i$ was sampled. Then, 
for iteration $k+1 = r\tau$,
\begin{equation}\label{Eq_scaffold}
\begin{aligned}
x_{i,k+1}&=\left( 1-\gamma _g \right) x_{\left( r-1 \right) \tau}
\\
&\quad+\frac{\gamma _g}{\left| S_{r-1} \right|}\sum_{j\in S_{r-1}}{\left( x_{j,(r-1)\tau}-\gamma _l\sum_{t=0}^{\tau -1}{y_{j,k-t}} \right)},
\\
y_{i,k+1}&=\left( 1-{\gamma_c} \right) c_{\left( r-1 \right) \tau}
\\
&\quad+\frac{\gamma_c}{\left| S_{r-1} \right|}\sum_{j\in S_{r-1}}{\frac{1}{\tau}\sum_{t=0}^{\tau -1}{y_{j,k-t}}}+g_{i,k+1}-c_{i,t\left( i \right)};
\end{aligned}
\end{equation}
for $k+1 = r\tau+1,\dots,r\tau+\tau-1$,
\begin{equation}
\begin{aligned}
x_{i,k+1}&=x_{i,k}-\gamma _ly_{i,k},
\\
y_{i,k+1}&=y_{i,k}+g_{i,k+1}-g_{i,k}.
\end{aligned}
\end{equation}
And, for $k+1=r\tau+\tau$,
\begin{equation}
c_{i,r\tau+\tau}=\frac{1}{\tau}\sum_{t=0}^{\tau -1}{g_{i,r\tau +t}}
\end{equation}
Moreover, for unsampled nodes $i \notin S_r$, all associated variables remain unchanged.
\end{prop}

\begin{proof}
At the beginning of the $r$-th round with $k+1 = r\tau$, each node $i\in S_r$ receives information from the server, which is updated based on the information uploaded by the nodes sampled in $S_{r-1}$, i.e.,
\begin{equation}
\begin{aligned}
x_{i,k+1}&=x_{k+1}
\\
&=x_k+\frac{\gamma _g}{\left| S_{r-1} \right|}\sum_{j\in S_{r-1}}{\left( x_{j,k+1/2}-x_k \right)}
\\
&=\left( 1-\gamma _g \right) x_{\left( r-1 \right) \tau}
\\
&\quad+\frac{\gamma _g}{\left| S_{r-1} \right|}\sum_{j\in S_{r-1}}{\left( x_{j,\left( r-1 \right) \tau}-\gamma _l\sum_{t=0}^{\tau -1}{y_{j,k-t}} \right)}
\\
&=\frac{1}{\left| S_{r-1} \right|}\sum_{j\in S_{r-1}}{\left( x_{j,\left( r-1 \right) \tau}-\gamma _g\gamma _l\sum_{t=0}^{\tau -1}{y_{j,k-t}} \right)}.
\end{aligned}
\end{equation}

At the end of the $r$-th round with $k+1=r\tau+\tau$, the local control variable will be updated, i.e.,
\begin{equation}
\begin{aligned}
c_{i,k+1}&=c_{i,k}-c_k+\frac{1}{\tau \gamma _l}\left( x_k-x_{i,k+1/2} \right) 
\\
&=c_{i,k}-c_k+\frac{1}{\tau \gamma _l}\left( x_k-x_{i,k}+\gamma _ly_{i,k} \right) 
\\
&=c_{i,k}-c_k+x_{r\tau}-x_{i,r\tau}+\frac{1}{\tau \gamma _l}\sum_{t=0}^{\tau -1}{\gamma _ly_{i,k-t}}
\\
&=\frac{1}{\tau}\sum_{t=0}^{\tau -1}{\left( c_{i,k-t}-c_{k-t}+y_{i,k-t} \right)}=\frac{1}{\tau}\sum_{t=0}^{\tau -1}{g_{i,r\tau+t},}
\end{aligned}
\end{equation}
where we used the facts that $x_{i,r\tau}=x_{r\tau}=\cdots =x_{r\tau+\tau-1}$, $c_{i,k-\tau +1}=\cdots =c_{i,k}$ and $c_{k-\tau +1}=\cdots =c_k$. This equation shows that the local control variable equals the running average of the local stochastic gradient within a period.

Then, for the global control variable, we have
\begin{equation}
\begin{aligned}
c_{r\tau}&=c_k+\frac{\gamma _c}{\left| S_{r-1} \right|}\sum_{j\in S_{r-1}}{\,\,\left( c_{j,k+1}-c_{j,k} \right)}
\\
&=\left( 1-\gamma _c \right) c_k+\frac{\gamma _c}{\left| S_{r-1} \right|}\sum_{j\in S_{r-1}}{\,\,\frac{1}{\tau}\sum_{t=0}^{\tau -1}{y_{j,\left( r-1 \right) \tau +t}}}.
\end{aligned}
\end{equation}
During the local updates phase, it is easy to get $y_{i,k+1}=y_{i,k}+g_{i,k+1}-g_{i,k}$.
\end{proof}

\subsection{Connection between Scaffold$^+$ and {\myalg}}

Based on Proposition~\ref{Prop_reform_scaffold}, we establish the connections between Scaffold$^+$ and {\myalg} from the perspective of random communication topology and gradient approximation scheme, respectively. To simplify notations, we assume $s = \left| S_{r} \right| \leqslant n$ for all $ r \in [R-1]$.

\begin{figure}[t]

\begin{minipage}[b]{0.48\linewidth}
  \centering
  \centerline{\includegraphics[width=0.9\linewidth]{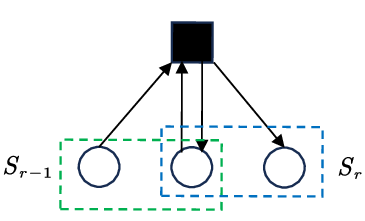}}
  \centerline{(a) Server--worker}\medskip
\end{minipage}
\begin{minipage}[b]{0.45\linewidth}
  \centering
  \centerline{\includegraphics[width=0.9\linewidth]{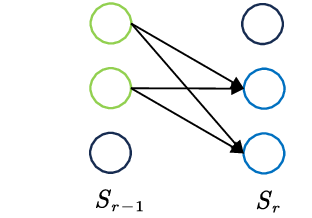}}
  \centerline{(b) Equivalent topology}\medskip
\end{minipage}
\caption{\hy{Illustration of the equivalent topology of Scaffold$^+$ under the random network perspective. Black solid rectangles represent server nodes, while circles represent worker nodes. Information is transmitted from nodes sampled at round $r-1$ (green box) to nodes sampled at time $r$ (blue box).}}
\label{fig:ps_2_de}
\end{figure}

{\textbf{Random network perspective.} 
Intuitively, the server in Scaffold$^+$ acts as a relay node, transmitting information from the workers in $S_{r-1}$ to those in $S_r$. The communication protocol can be described from the perspective of a bipartite graph, as illustrated in Fig.~\ref{fig:ps_2_de}. In particular, according to Proposition~\ref{Prop_reform_scaffold} with $\gamma_c=1$ and noticing that the variables of the unsampled workers are unchanged, we get an equivalent weight matrix at round $r$ as follows:
\begin{equation}\label{Eq_Wr}
\begin{aligned}
W_r = \frac{1}{s}\boldsymbol{e}_{S_r}\boldsymbol{e}_{S_{r-1}}^{\top}+\mathbf{I}-\mathrm{diag}\left( \boldsymbol{e}_{S_{r}} \right),
\end{aligned}
\end{equation}
where $\boldsymbol{e}_{S_r}$ is an $n$-dimensional column vector where the index of the sampled nodes $S_r$ is 1, and all other elements are 0. Note that $W_r$ is row-stochastic. Assuming that workers are sampled uniformly and independently at each round, the expectation of $W_r$ is
\begin{equation}
\begin{aligned}
\mathbb{E} \left[ W_r \right] = \frac{s}{n}\mathbf{J}+\frac{n-s}{n}\mathbf{I},
\end{aligned}
\end{equation}
which corresponds to a fully connected graph with spectral gap $1-s/n$. }
\hy{This random network perspective provides a foundation for the unified analysis of algorithms under server–worker architectures with partial participation and distributed topologies.}

{\textbf{Gradient tracking perspective.} As shown in \eqref{Eq_scaffold}, setting $\gamma_c = 1$, 
Scaffold$^+$ has the same spatio-temporal gradient-tracking scheme as that of  {\myalg}. Particularly, for $i\in S_r$, we have 
\begin{equation}
\begin{aligned}
y_{i,r\tau}=\frac{1}{s}\sum_{j\in S_{r-1}}{\frac{1}{\tau}\sum_{t=0}^{\tau -1}{y_{j,\left( r-1 \right) \tau +t}}}+g_{i,r\tau}-c_{i,t\left( i \right)};
\end{aligned}
\end{equation}
and for $i\notin S_r$, $y_{i,r\tau}=y_{i,r\tau -1}=\cdots =y_{i,\left( r-1 \right) \tau}$. By proper initialization, we have
\begin{equation}
\begin{aligned}
\frac{1}{n}\sum_{i=1}^n{\frac{1}{\tau}\sum_{t=0}^{\tau -1}{\mathbb{E} \left[ y_{i,r\tau +t} \right]}}=\frac{s}{n}\frac{1}{n}\sum_{i=1}^n{\frac{1}{\tau}\sum_{t=0}^{\tau -1}{\mathbb{E} \left[ g_{i,r\tau +t} \right]},}
\end{aligned}
\end{equation}
illustrating the similar gradient tracking property as in {\myalg}. Note that the contribution from unsampled nodes is slightly different compared to {\myalg}. \hy{Instead, by setting $\gamma_c = s/n$ as in the original Scaffold (cf. line 13 in Algorithm \ref{alg:scaffold}) and noting that $c_k$ stores the average of the most recent $c_{i,k}$ values and is updated every $\tau$ iterations, we obtain}
\begin{equation}
\begin{aligned}
c_{r\tau}=\frac{\mathbf{1}^{\top}}{n}\left( \left[ \begin{array}{c}
	c_{1,t\left( 1 \right)}\\
	\vdots\\
	c_{n,t\left( n \right)}\\
\end{array} \right] +\left[ \begin{array}{c}
	\vdots\\
	c_{j,r\tau}-c_{j,t\left( j \right)}\\
	\vdots\\
\end{array} \right] _{j\in S_{r-1}} \right) ,
\end{aligned}
\end{equation}
which \hy{serves as the same role as the full gradient table in SAGA} \citep{defazio2014saga}.

These observations demonstrate that Scaffold for federated learning can be interpreted as a spatio-temporal gradient tracking method operating over random networks. In particular, when $s = n$, Scaffold becomes exactly equivalent to {\myalg} with $W_r = \mathbf{J}$. This perspective further offers a unified framework for understanding other federated learning algorithms, such as FedAvg.

\section{Convergence Results}
\label{sec:conv_results}
In this section, we establish the convergence properties of {\myalg} under several standard assumptions on the objective functions and their gradients.

\subsection{Assumptions}

\begin{assum}[Convexity]\label{Ass_cov}
Each $f_{i}: \mathbb{R}^p\rightarrow \mathbb{R}$ is $\mu$-strongly convex, i.e., for any $x, x'\in \mathbb{R}^p$, there exists a constant $\mu > 0$ such that
\begin{equation}
\begin{aligned}
\left< \nabla f_i\left( x \right) -\nabla f_i\left( x^{\prime} \right), x-x^{\prime} \right> \geqslant \mu \left\| x-x^{\prime} \right\| ^2.
\end{aligned}
\end{equation}
\end{assum}

\begin{assum}[Smoothness]\label{Ass_smoothness}
Each $f_{i}: \mathbb{R}^p\rightarrow \mathbb{R}$ is $L$-smooth, i.e., for any $x, x'\in \mathbb{R}^p$,  there exists a constant $L > 0$ such that
\begin{equation}
\begin{aligned}
& \left\| \nabla f_i\left( x \right) -\nabla f_i\left( x' \right) \right\|\leqslant L\left\| x-x' \right\|.
\end{aligned}
\end{equation}
\end{assum}

\begin{assum}
[Bounded variance]\label{Ass_sampling}
For each node $i$, the stochastic gradient is unbiased, i.e., $\mathbb{E} _{\xi _i\sim \mathcal{D} _i}\left[ \nabla f_i\left( x;\xi _i \right) \right] =\nabla f_i\left( x \right) , \forall x\in \mathbb{R} ^p$, and there exists a constant $\sigma \geqslant0$ such that
\begin{equation}\label{Def_noise}
\begin{aligned}
\mathbb{E}_{\xi _i\sim \mathcal{D} _i} \left[ \left\| \nabla f_i\left( x;\xi _i \right) -\nabla f_i\left( x \right) \right\| ^2 \right] \leqslant \sigma ^2.
\end{aligned}
\end{equation}
\end{assum}

\subsection{Strongly convex Case}
We are now ready to give the convergence results of the {\myalg} algorithm.
To this end, we first define a Lyapunov function as follows:
\begin{equation}\label{Eq_Lyapunov}
\begin{aligned}
V_k=\left\| \bar{x}_k-x^* \right\| ^2+c_x\left\| X_k-\mathbf{1}\bar{x}_k \right\| ^2+c_y\left\| \varUpsilon _{k} \right\| ^2,
\end{aligned}
\end{equation}
where 
\begin{equation}
\varUpsilon _{k}:=Y_{k}-G_{k}-\nabla f\left( \mathbf{1}\bar{x}_{k} \right) +\nabla F\left( \mathbf{1}\bar{x}_{k} \right), 
\end{equation}
and the coefficients $c_x$ and $c_y$ are designed as:
\begin{equation}\label{Eq_Lya_coff}
\begin{aligned}
c_x=\frac{80\gamma \tau L}{n\left( 1-\rho \right)},\quad  c_y=\frac{3556\gamma ^3\tau^3L}{n\left( 1-\rho \right) ^3}.
\end{aligned}
\end{equation}

Then, for the strongly convex objective functions, we have the following theorem.
\begin{thm}[Strongly convex case]\label{Thm_sc}
Suppose Assumptions~\ref{Ass_graph}--\ref{Ass_sampling} hold. Let the stepsize 
$\gamma =\mathcal{O} \left( \frac{\left( 1-\rho \right) ^2}{\tau L} \right)$.
Then, we have
\begin{equation}\label{Eq_thm_sc_conv}
\begin{aligned}
\mathbb{E} \left[ \left\| V_{\left( r+1 \right) \tau} \right\| ^2 \right] &\leqslant \left( 1-\min \left\{ \frac{\mu \gamma}{2},\frac{1-\rho}{8} \right\} \right) \mathbb{E} \left[ \left\| V_{r\tau} \right\| ^2 \right] 
\\
&\quad+\mathcal{O} \left( \gamma ^2\tau\frac{\sigma ^2}{n}+\frac{\gamma ^3\tau^2L}{\left( 1-\rho \right) ^3}\sigma ^2 \right).
\end{aligned}
\end{equation}
Further, the {\myalg} algorithm achieves an accuracy of $\epsilon>0$ after at least the following rounds of communications:
\begin{equation}\label{Eq_thm_sc_comm}
\begin{aligned}
R=\tilde{\mathcal{O}}\left( \frac{L}{\left( 1-\rho \right) ^2\mu}+\frac{\sigma ^2}{\mu ^2n\tau \epsilon}+\sqrt{\frac{L\sigma ^2}{\mu ^3\left( 1-\rho \right) ^3\tau \epsilon}} \right) .
\end{aligned}
\end{equation}

\end{thm}

\begin{proof}
    See Appendix~\ref{Subsec_proof_Thm_sc}.
\end{proof}

\begin{rem}
Theorem~\ref{Thm_sc} shows that {\myalg} converges linearly to a neighborhood of the optimal solution for strongly-convex and smooth objectives. The neighborhood size has two components: one matching the centralized SGD algorithm, which achieves linear speedup with respect to $n$, and the other induced by the network topology. More importantly, the communication complexity of {\myalg} achieves a linear speedup with respect to the number of local updates $\tau$, and is significantly lower than that of the FlexGT algorithm. Specifically, the topology-dependent term in \eqref{Eq_thm_sc_comm} is scaled by $1/\tau$, thereby reducing the communication cost, particularly in cases of weak network connectivity, i.e., when $\rho \to 1$. Furthermore, in the absence of the gradient noise, {\myalg} attains exact linear convergence. A more detailed comparison with existing related algorithms in terms of computation, storage, and communication complexity is summarized in Table~\ref{Tab_comparison_of_related_work}.
\end{rem}

\subsection{Nonconvex Case}
When the objective functions are not convex, the following theorem shows a sublinear convergence rate of {\myalg}.

\begin{thm}[Nonconvex case]\label{Thm_nc}
Suppose Assumptions~\ref{Ass_graph},~\ref{Ass_smoothness} and~\ref{Ass_sampling} hold. Let the stepsize satisfy 
\begin{equation}
\gamma \leqslant \min \left\{ \frac{1-\rho}{178\tau L}, \frac{\left( 1-\rho \right) ^2}{625\sqrt{\rho}\tau L} \right\}.
\end{equation}
Then, for the {\myalg} algorithm, we have
\begin{equation}\label{Eq_thm_nc_conv}
\begin{aligned}
&\frac{1}{nR}\sum_{r=0}^{R-1}{\mathbb{E} \left[ \left\| \nabla F\left( \mathbf{1}\bar{x}_{r\tau} \right) \right\| ^2 \right]}
\\
&\leqslant \frac{16\left( f\left( \bar{x}_0 \right) -f\left( \bar{x}_{R\tau} \right) \right)}{\gamma \tau R}+\frac{\mathbb{E} \left[ \left\| \varUpsilon _0 \right\| ^2 \right]}{nR}
\\
&\quad +\mathcal{O} \left( \gamma L\frac{\sigma ^2}{n}+\frac{\tau \gamma ^2L^2\sigma ^2}{\left( 1-\rho \right) ^3} \right).
\end{aligned}
\end{equation}
Further, the {\myalg} algorithm achieves an accuracy of $\epsilon>0$ to a stationary point after at least the following rounds of communications
\begin{equation}\label{Eq_thm_nc_comm}
\begin{aligned}
R=\mathcal{O} \left( \frac{L}{\left( 1-\rho \right) ^2\epsilon}+\frac{\sigma ^2L}{n\tau \epsilon ^2}+\frac{L\sigma}{\sqrt{\tau \left( 1-\rho \right) ^3\epsilon ^3}}+\frac{C_0}{\epsilon} \right) , 
\end{aligned}
\end{equation}
where $C_0:=\frac{1}{n}\mathbb{E} \left[ \left\| \varUpsilon _0 \right\| ^2 \right] \left( f\left( \bar{x}_0 \right) -f^* \right) ^{-1}$.
\end{thm}

\begin{proof}
    See Appendix~\ref{Subsec_proof_Thm_nc}.
\end{proof}

\begin{rem}
    \hy{Theorem 2 shows that {\myalg} converges to a neighborhood of a local optimum at a sublinear rate, where the neighborhood size depends on the gradient noise level, objective properties, and network connectivity. Compared with K-GT \citep{liu2024decentralized}, {\myalg} reduces the network dependence from $1/(1-\rho)^2$ to $1/(1-\rho)^{3/2}$ and achieves scale invariance to gradient noise by a factor of $\tau$.}
\end{rem}

\section{Numerical Results}
\label{sec:simulation}

\subsection{Synthetic Example}
To validate our theoretical findings and illustrate the effectiveness of {\myalg}, we consider the following distributed ridge regression problem over a network of $n=32$ nodes:
\begin{equation}\label{Prob_RR}
\begin{aligned}
\underset{x\in \mathbb{R}^p}{\min}f\left( x \right) =\frac{1}{n}\sum_{i=1}^n{\underset{=:f_i}{\underbrace{\left( \mathbb{E} _{d_i}\left[ \left( \theta_{i}^{\top}x-d_i \right) ^2+\frac{\mu}{2}\left\| x \right\| ^2 \right] \right) }}},
\end{aligned}
\end{equation}
where $\mu>0$ is the regularization parameter, $\theta_i\in \left[ 0,1 \right] ^p$ denotes the feature parameters of node $i$ with dimension $p=10$, and $d_i\sim \mathcal{N} \left( \bar{d}_i,\sigma ^2 \right)$ with $\bar{d}_i\in \left[ 0,1 \right]$. The algorithms can obtain an unbiased noisy gradient $\nabla f_i\left( x_{i,k} \right) +\delta _{i,k}$ with $\delta _{i,k}\sim \mathcal{N} \left( 0,\sigma^2 \right) $ at each iteration $k$.

We compare the convergence performance of FlexGT, Scaffold, and {\myalg} in terms of the residual $\left\| \bar{x}_k - x^* \right\|^2$ as shown in Fig.~\ref{fig:res}. For the communication topology, Scaffold samples $s=\{4, 16\}$ nodes at each round, while in FlexGT and {\myalg}, each node in the exponential graph has $s-1$ neighbors. The number of local updates is set to $\tau=50$, the stepsize to $\gamma=0.4$, and the noise variance to $\sigma^2=0.1$ for all algorithms.
The results show that {\myalg} achieves the lowest steady-state error. While Scaffold attains a smaller error than FlexGT, its performance remains inferior to {\myalg} due to the uncertainty introduced by its random topology. Moreover, as the number of communicating nodes increases from 4 to 16, the performance of all algorithms approaches that of the centralized setting, and the gap among the three algorithms becomes smaller.

\hy{To illustrate the effect of local update frequency on {\myalg}, Fig.~\ref{fig:synthetic_impact_of_tau} shows the convergence behavior measured in communication rounds under different network topologies for $\tau = 25$, $50$, and $100$. The stepsizes follow the proportional relationship recommended in Theorem~\ref{Thm_sc}, with $\gamma = 0.8$, $0.4$, and $0.2$. The results confirm that {\myalg} achieves an almost linear speed-up with respect to $\tau$.}

\subsection{\hy{Training ResNet-18 on CIFAF-10}}

We further evaluate the performance of {\myalg} on the real-world dataset CIFAR-10 \citep{krizhevsky2009learning}. Specifically, we perform distributed training of ResNet-18 using multiple processes to emulate nodes on a single A40 GPU, with the Gloo backend handling inter-node communication. The training data are unevenly partitioned across eight nodes, where each node contains samples from only eight of the ten classes, leading to heterogeneous local datasets. The learning rate is set to 1, the batch size is 200, and the number of local updates is set to $\tau=25$.
We compare the training loss and testing accuracy of FlexGT, Scaffold, and {\myalg} across communication rounds. Figure~\ref{fig:cifar} shows that {\myalg} achieves the best performance in both training and testing. Scaffold, which also incorporates spatio-temporal gradient tracking, attains the second-best results, whereas FlexGT, relying solely on local updates, performs the worst. These results further corroborate our theoretical analysis and confirm the effectiveness of the proposed algorithm.

\begin{figure}[t]

\begin{minipage}[b]{0.49\linewidth}
  \centering
  \centerline{\includegraphics[width=1\linewidth]{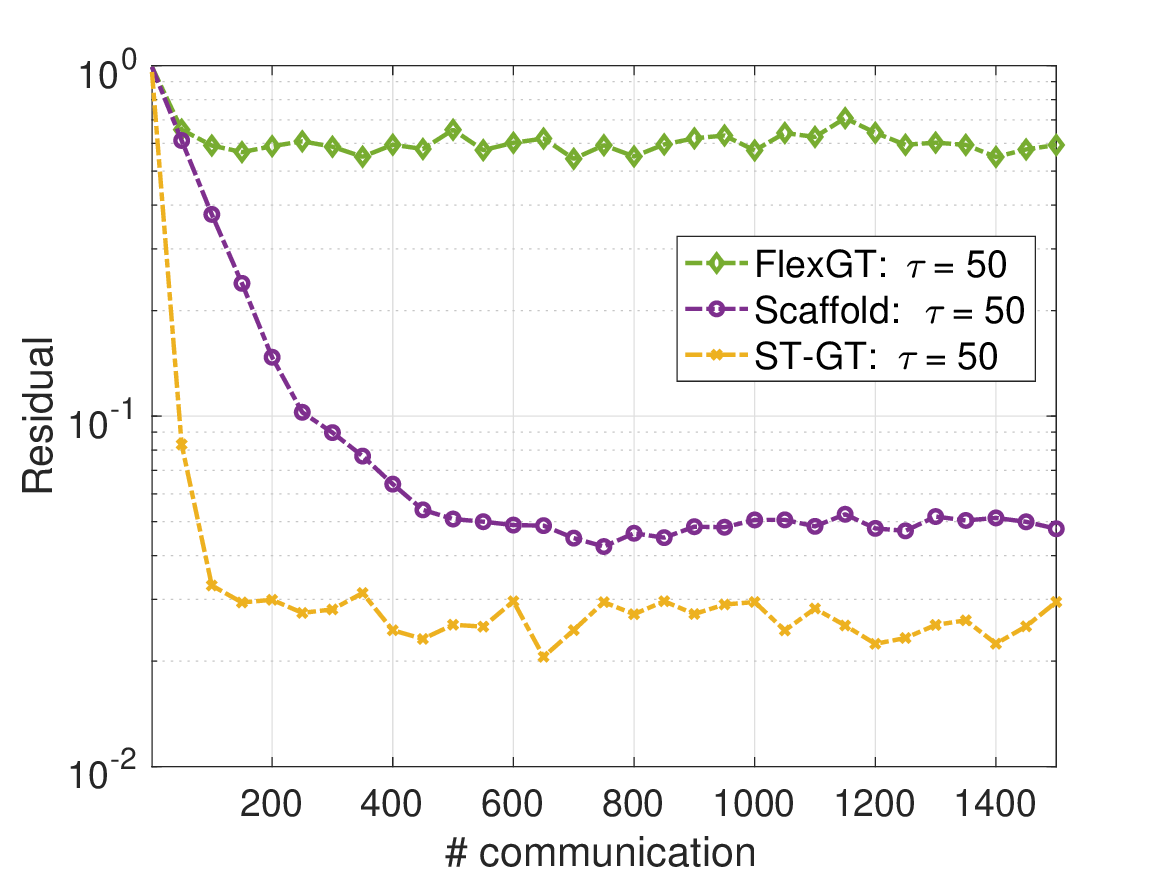}}
\end{minipage}
\begin{minipage}[b]{0.49\linewidth}
  \centering
  \centerline{\includegraphics[width=1\linewidth]{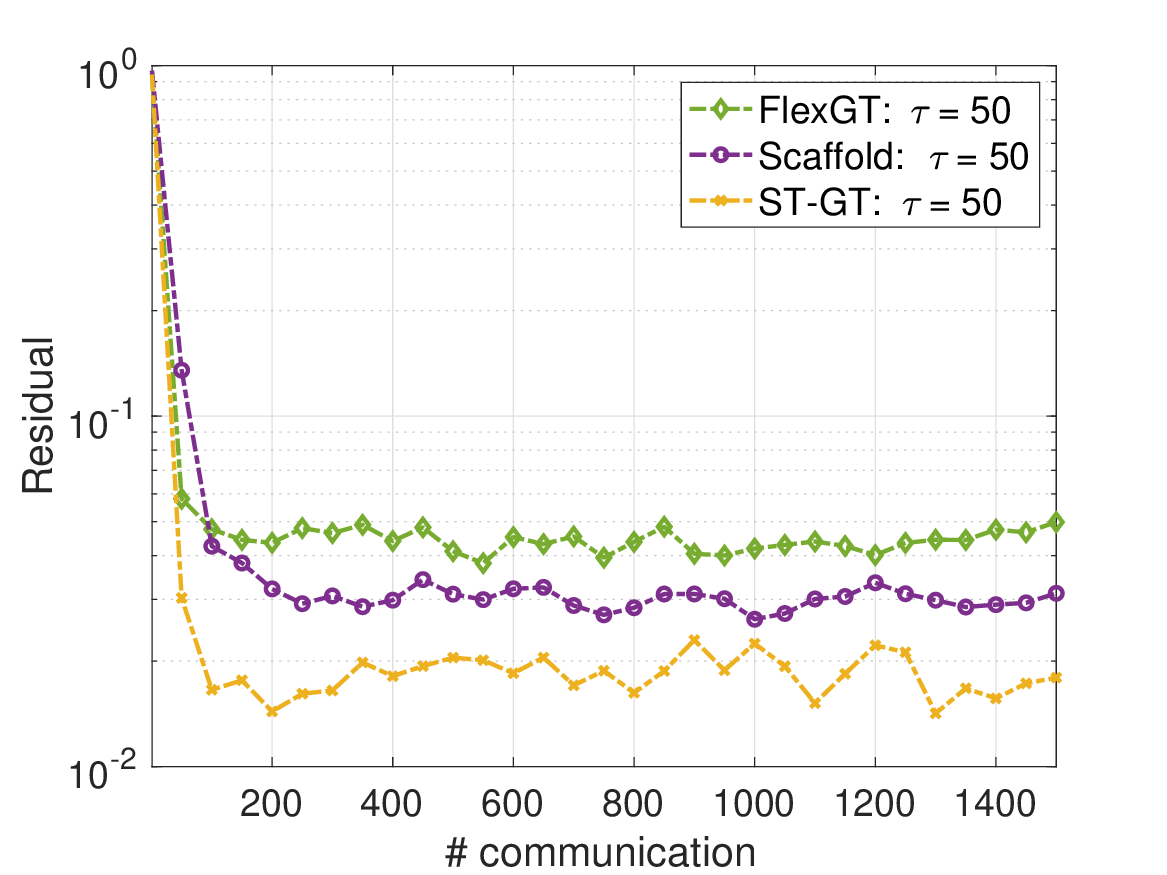}}
\end{minipage}
\caption{Comparison of the convergence between Scaffold, FlexGT, and {\myalg}. The number of nodes is $n=32$. For Scaffold with partial node participation, we set $s=4$ (left) and $s=16$ (right). For the exponential graph used in the other two algorithms, each node is connected to $3$ and $15$ neighbors, respectively. }
\label{fig:res}
\end{figure}

\begin{figure}[t]
\begin{minipage}[b]{0.49\linewidth}
  \centering
  \centerline{\includegraphics[width=1\linewidth]{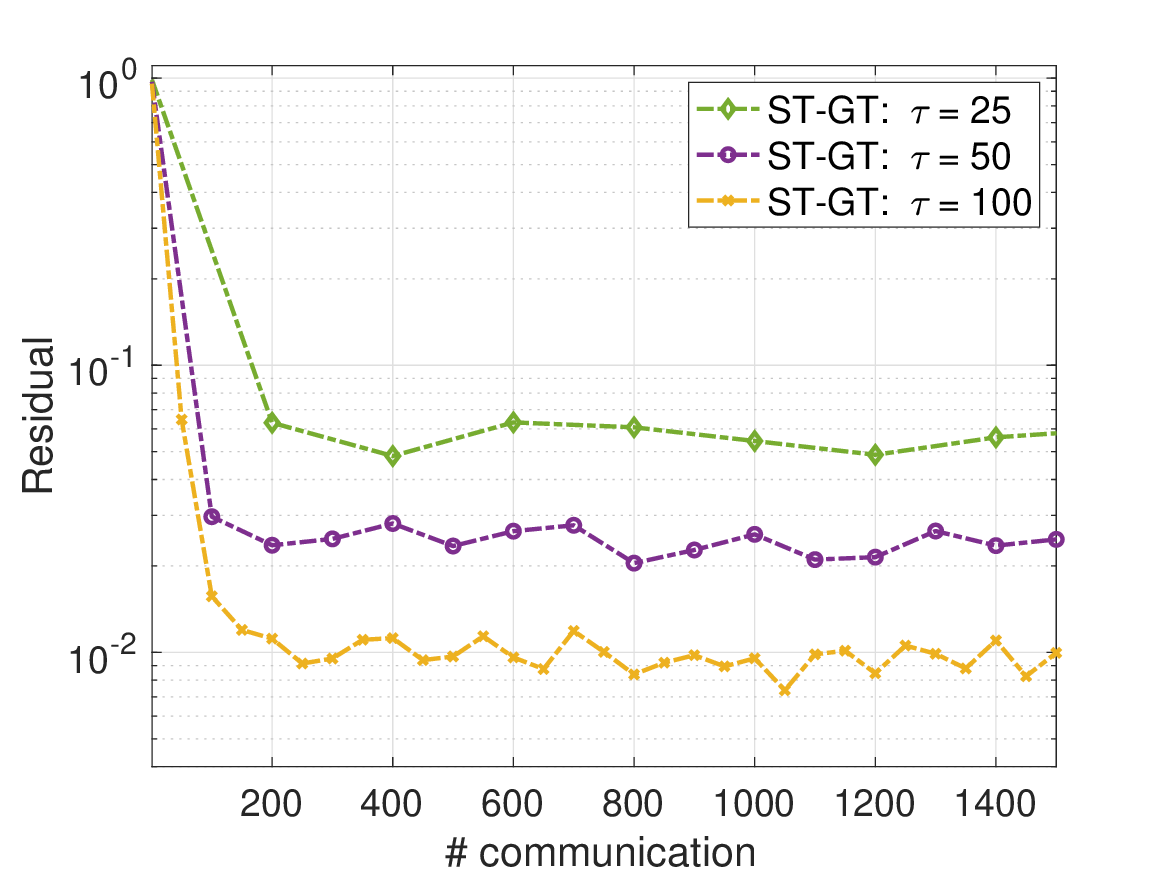}}
\end{minipage}
\begin{minipage}[b]{0.49\linewidth}
  \centering
  \centerline{\includegraphics[width=1\linewidth]{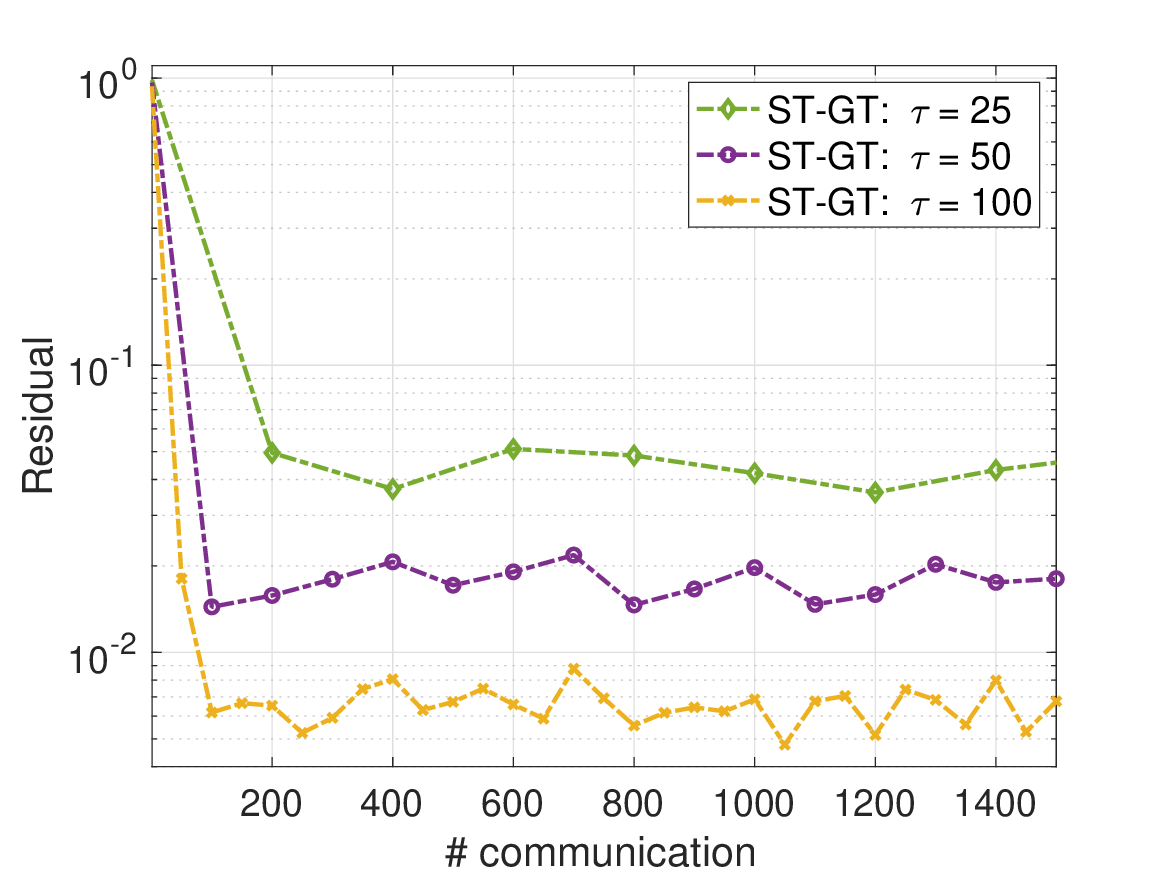}}
\end{minipage}
\caption{\hy{Impact of the number of local updates $\tau$. We use an exponential graph with $n=32$ nodes, where each node is connected to $3$ (left) or $15$ (right) neighbors.}}
\label{fig:synthetic_impact_of_tau}
\end{figure}

\begin{figure}[t]

\begin{minipage}[b]{0.49\linewidth}
  \centering
  \centerline{\includegraphics[width=1\linewidth]{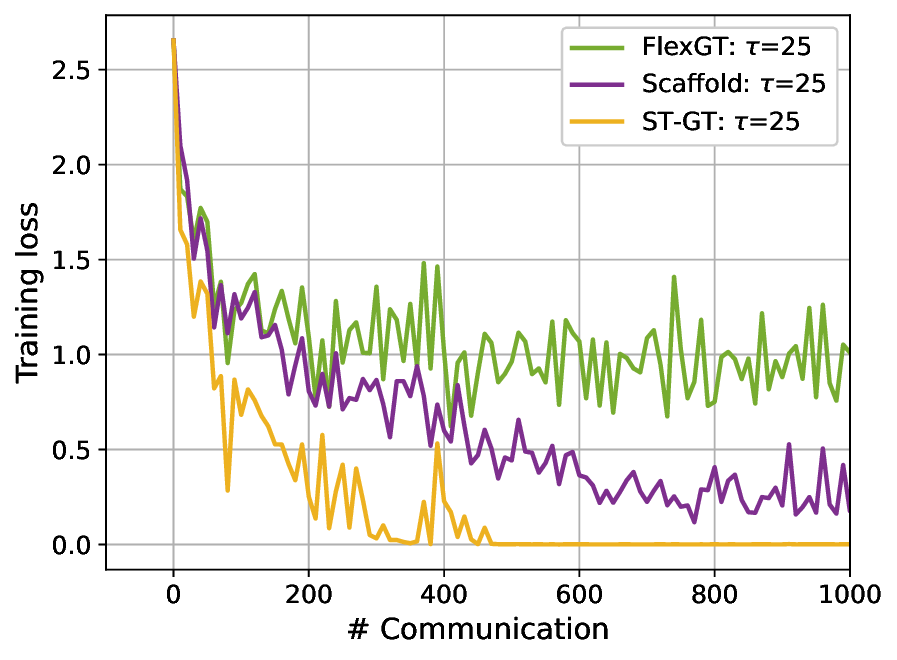}}
\end{minipage}
\begin{minipage}[b]{0.48\linewidth}
  \centering
  \centerline{\includegraphics[width=1\linewidth]{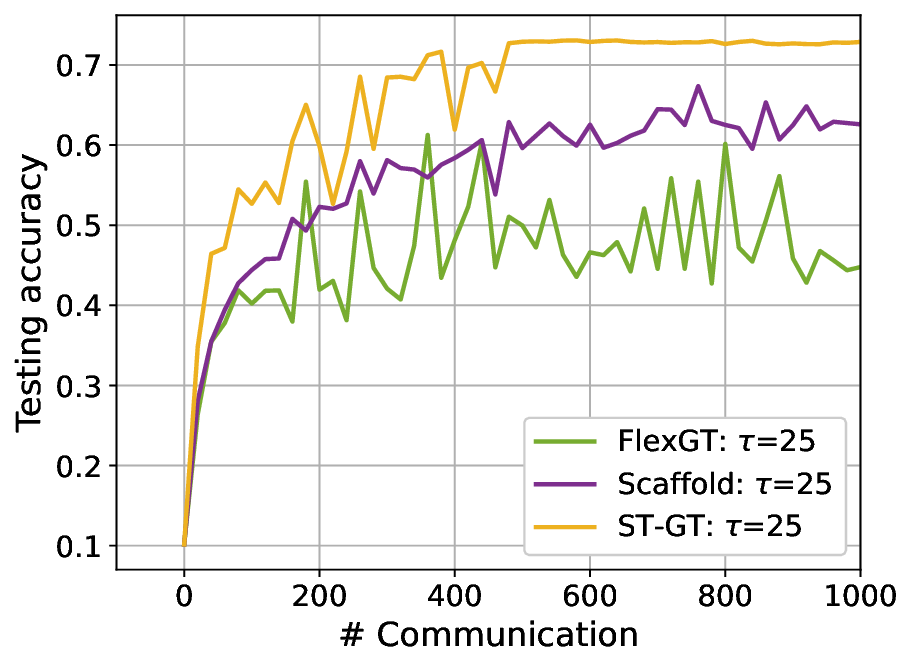}}
\end{minipage}
\caption{\hy{Comparison among {\myalg}, FlexGT, and Scaffold algorithms for distributed training of ResNet-18 on the CIFAR-10 dataset with $n=8$ nodes. The plots illustrate the training loss (left) and test accuracy (right) as functions of the communication rounds.}}
\label{fig:cifar}
\end{figure}

\section{Conclusion}
\label{sec:conclusion}

We have proposed a unified spatio-temporal gradient tracking algorithm, {\myalg}, for distributed stochastic optimization with non-i.i.d. datasets. By simultaneously tracking the global gradient across nodes and the time-averaged local stochastic gradients at each node, {\myalg} has improved robustness against data heterogeneity and mitigated the effect of the gradient noise. By investigating an extended version of Scaffold, we further revealed that it could also be interpreted as a spatio-temporal gradient tracking method. Without assuming any data similarity, we proved that {\myalg} achieved a linear and a sublinear convergence rate for strongly convex and nonconvex objective functions, respectively, while significantly reducing communication complexity compared with the FlexGT and K-GT algorithms, especially under poor network connectivity. Simulation results corroborated the theoretical analysis and demonstrated the effectiveness of {\myalg}.

\bibliography{ifacconf}       

\appendix
\section{Proof of the main results}
\label{sec:appendix}

In this section, we provide the detailed convergence analysis for the {\myalg} algorithm.

\subsection{Supporting Lemmas}
\hy{We first bound the extent to which the local model parameters diverge from the averaged model between two communication rounds.}
\begin{lem}[Client divergence within a period]\label{Lem_client_diver}
Suppose Assumptions~\ref{Ass_graph}-\ref{Ass_sampling} hold. Let the stepsize satisfy $\,\,\gamma \leqslant \frac{1}{8\tau L}$. We have for all $t\in[0,\tau-1]$,
\begin{equation}
\begin{aligned}
&\frac{1}{n}\sum_{t=0}^{\tau -1}{\mathbb{E} \left[ \left\| X_{r\tau +t}-\mathbf{1}\bar{x}_{r\tau} \right\| ^2 \right]}
\\
&\leqslant 3\tau \frac{1}{n}\mathbb{E} \left[ \left\| X_{r\tau}-\mathbf{1}\bar{x}_{r\tau} \right\| ^2 \right] +\frac{9}{2}\gamma ^2\tau ^3\frac{1}{n}\mathbb{E} \left[ \left\| \varUpsilon _{r\tau} \right\| ^2 \right] 
\\
&\quad+\frac{9}{2}\gamma ^2\tau ^3\mathbb{E} \left[ \left\| \nabla f\left( \bar{x}_{r\tau} \right) \right\| ^2 \right] +\frac{3\gamma ^2\tau ^2}{2}\sigma ^2.
\end{aligned}
\end{equation}
\end{lem}

\begin{proof}
By the update rule of {\myalg} in \eqref{Eq_ST_GT}, we have
\begin{equation}
\begin{aligned}
&\mathbb{E} \left[ \left\| x_{i,r\tau+t}-\bar{x}_{r\tau} \right\| ^2 \right] 
\\
&=\mathbb{E} \left[ \left\| x_{i,r\tau+t-1}-\gamma \left( y_{i,r\tau}+g_{i,r\tau+t}-g_{i,r\tau} \right) -\bar{x}_{r\tau} \right\| ^2 \right].
\end{aligned}
\end{equation}
Then, by Assumptions~\ref{Ass_cov} and~\ref{Ass_sampling}, and using Young's inequality with parameter $\beta>0$, we get
\begin{equation}\label{Eq_divergence_1}
\begin{aligned}
&\mathbb{E} \left[ \left\| X_{r\tau +t}-\mathbf{1}\bar{x}_{r\tau} \right\| ^2 \right] 
\\
&\leqslant \left( 1+\beta \right) \mathbb{E} \left[ \left\| X_{r\tau +t-1}-\mathbf{1}\bar{x}_{r\tau} \right\| ^2 \right] +\gamma ^2n\sigma ^2
\\
&\quad+\left( 1+\beta ^{-1} \right) \gamma ^2\mathbb{E} \left[ \left\| Y_{r\tau}-G_{r\tau}+\nabla F_{rK+t-1} \right\| ^2 \right] 
\\
&\leqslant \left( 1+\beta +3\left( 1+\beta ^{-1} \right) \gamma ^2L^2 \right) \mathbb{E} \left[ \left\| X_{r\tau +t-1}-\mathbf{1}\bar{x}_{r\tau} \right\| ^2 \right] 
\\
&\quad+3\left( 1+\beta ^{-1} \right) \gamma ^2\mathbb{E} \left[ \left\| \varUpsilon _{r\tau} \right\| ^2 \right] 
\\
&\quad+3\left( 1+\beta ^{-1} \right) \gamma ^2\mathbb{E} \left[ \left\| \nabla f\left( \bar{x}_{r\tau} \right) \right\| ^2 \right] +\gamma ^2 n \sigma ^2.
\end{aligned}
\end{equation}
Letting $\beta =\frac{1}{\tau}$, $\,\,\gamma \leqslant \frac{1}{8\tau L}$ and noticing that
\[
\left( 1+\beta +3\left( 1+\beta ^{-1} \right) \gamma ^2L^2 \right) ^t<3,
\]
we obtain the result by iteratively applying \eqref{Eq_divergence_1}.
\end{proof}

\hy{Building on the above lemma, we can establish the contraction property of the consensus error.}

\begin{lem}[Consensus error]\label{Lem_cons_err}
Suppose Assumptions~\ref{Ass_graph}-\ref{Ass_sampling} hold. Let the stepsize satisfy $\gamma \leqslant \min \left\{ \frac{1-\rho}{11\tau L\sqrt{\rho}}\,\,,\frac{1}{8\tau L}\,\, \right\} $. We have
\begin{equation}\label{Eq_cons_err}
\begin{aligned}
&\mathbb{E} \left[ \left\| X_{\left( r+1 \right) \tau}-\mathbf{1}\bar{x}_{\left( r+1 \right) \tau} \right\| ^2 \right] 
\\
&\leqslant \frac{3+\rho}{4}\mathbb{E} \left[ \left\| X_{r\tau}-\mathbf{1}\bar{x}_{r\tau} \right\| ^2 \right] +\frac{5\rho}{1-\rho}\gamma ^2\tau ^2\mathbb{E} \left[ \left\| \varUpsilon _{r\tau} \right\| ^2 \right] 
\\
&\quad+3\rho \gamma ^2\tau n\sigma ^2+\frac{9\rho}{1-\rho}\gamma ^2\tau ^2n\mathbb{E} \left[ \left\| \nabla f\left( \bar{x}_{r\tau} \right) \right\| ^2 \right].
\end{aligned}
\end{equation}
\end{lem}

\begin{proof}
By the update rule of {\myalg} \eqref{Eq_ST_GT} and under Assumptions~\ref{Ass_graph} and~\ref{Ass_cov}, we have
\begin{equation}
\begin{aligned}
&\mathbb{E} \left[ \left\| X_{\left( r+1 \right) \tau}-\mathbf{1}\bar{x}_{\left( r+1 \right) \tau} \right\| ^2 \right] 
\\
&\leqslant \frac{1+\rho}{2}\mathbb{E} \left[ \left\| X_{r\tau}-\mathbf{1}\bar{x}_{r\tau} \right\| ^2 \right] 
\\
&\quad+\frac{2\left( 1+\rho \right) \rho}{1-\rho}\gamma ^2\tau ^2\mathbb{E} \left[ \left\| \varUpsilon _{r\tau} \right\| ^2 \right] +\rho \gamma ^2\tau n\sigma ^2
\\
&\quad+\frac{4\left( 1+\rho \right) \rho}{1-\rho}\gamma ^2\tau L^2\sum_{t=0}^{\tau -1}{\mathbb{E} \left[ \left\| X_{r\tau +t}-\mathbf{1}\bar{x}_{r\tau} \right\| ^2 \right]}
\\
&\quad+\frac{4\left( 1+\rho \right) \rho}{1-\rho}\gamma ^2\tau ^2n\mathbb{E} \left[ \left\| \nabla f\left( \bar{x}_{r\tau} \right) \right\| ^2 \right] 
\\
&\quad+2\mathbb{E} \left[ \left< \sum_{t=0}^{\tau -1}{\left( G\left( X_{r\tau +t} \right) -\nabla F_{r\tau +t} \right) \,\,},\sum_{t=0}^{\tau -1}{\nabla F_{r\tau +t}\,\,} \right> \right] .
\end{aligned}
\end{equation}
With the help of Lemma~\ref{Lem_client_diver}, we get
\begin{equation}
\begin{aligned}
&\mathbb{E} \left[ \left\| X_{\left( r+1 \right) \tau}-\mathbf{1}\bar{x}_{\left( r+1 \right) \tau} \right\| ^2 \right] 
\\
&\leqslant \left( \frac{1+\rho}{2}+\frac{30\rho}{1-\rho}\gamma ^2\tau ^2L^2 \right) \mathbb{E} \left[ \left\| X_{r\tau}-\mathbf{1}\bar{x}_{r\tau} \right\| ^2 \right] 
\\
&\quad+\left( \frac{4\rho}{1-\rho}\gamma ^2\tau ^2+\frac{45\rho}{1-\rho}\gamma ^4\tau ^4L^2 \right) \mathbb{E} \left[ \left\| \varUpsilon _{r\tau} \right\| ^2 \right] 
\\
&\quad+\left( \frac{8\rho}{1-\rho}\gamma ^2\tau ^2+\frac{45\rho}{1-\rho}\gamma ^4\tau ^4L^2 \right) n\mathbb{E} \left[ \left\| \nabla f\left( \bar{x}_{r\tau} \right) \right\| ^2 \right] 
\\
&\quad+2\rho \gamma ^2\tau n\sigma ^2+\frac{15\rho}{1-\rho}\gamma ^4\tau ^3L^2n\sigma ^2.
\end{aligned}
\end{equation}

Letting the stepsize satisfy $\gamma \leqslant \min \left\{ \frac{1-\rho}{11L\sqrt{\rho}}\,\,,\frac{1}{8\tau L}\,\, \right\} $, we obtain the result.
\end{proof} 

\hy{Similarly, we prove the contraction property of the gradient tracking error.}

\begin{lem}[Gradient-tracking error]\label{Lem_tracking_err}
Suppose Assumptions~\ref{Ass_graph}-\ref{Ass_sampling} hold. Let the stepsize $\gamma \leqslant \frac{1-\rho}{12\tau L}$. Then, we have
\begin{equation}\label{Eq_tracking_err}
\begin{aligned}
&\mathbb{E} \left[ \left\| \varUpsilon _{\left( r+1 \right) \tau} \right\| ^2 \right] 
\\
&\leqslant \frac{3+\rho}{4}\mathbb{E} \left[ \left\| \varUpsilon _{r\tau} \right\| ^2 \right] +\frac{27\left( 1+\rho \right) L^2}{1-\rho}\mathbb{E} \left[ \left\| X_{r\tau}-\mathbf{1}\bar{x}_{r\tau} \right\| ^2 \right] 
\\
&\quad+\frac{8\left( 1+\rho \right)}{1-\rho}n\mathbb{E} \left[ \left\| \nabla f\left( \bar{x}_{r\tau} \right) \right\| ^2 \right] +7n\frac{\sigma ^2}{\tau}.
\end{aligned}
\end{equation}
\end{lem}

\begin{proof}
\hy{By the update rule of {\myalg} in \eqref{Eq_ST_GT}}, we have
\begin{equation}
\begin{aligned}
\varUpsilon _{\left( r+1 \right) \tau}&=\left( W-\mathbf{J} \right) \varUpsilon _{r\tau}-W\left( \nabla F\left( \mathbf{1}\bar{x}_{r\tau} \right) -\nabla f\left( \mathbf{1}\bar{x}_{r\tau} \right) \right) 
\\
&\quad+\left( W-\mathbf{I} \right) \frac{1}{\tau}\sum_{t=0}^{\tau -1}{G_{r\tau +t}}
\\
&\quad+\left( \nabla F\left( \mathbf{1}\bar{x}_{\left( r+1 \right) \tau} \right) -\nabla f\left( \mathbf{1}\bar{x}_{\left( r+1 \right) \tau} \right) \right).
\end{aligned}
\end{equation}
Then, by Assumptions~\ref{Ass_graph}-\ref{Ass_sampling} and using Young's inequality, we get
\begin{equation}
\begin{aligned}
&\mathbb{E} \left[ \left\| \varUpsilon _{\left( r+1 \right) \tau} \right\| ^2 \right] 
\\
&\leqslant \frac{1+\rho}{2}\mathbb{E} \left[ \left\| \varUpsilon _{r\tau} \right\| ^2 \right] +6n\frac{\sigma ^2}{\tau}
\\
&\quad+\left( 3n\frac{\left( 1+\rho \right) L ^2}{1-\rho}+1 \right) \mathbb{E} \left[ \left\| \bar{x}_{\left( r+1 \right) \tau}-\bar{x}_{rK} \right\| ^2 \right] 
\\
&\quad+2\left( \frac{3\left( 1+\rho \right) L^2}{1-\rho}+1 \right) \frac{1}{\tau}\sum_{t=0}^{\tau -1}{\mathbb{E} \left[ \left\| X_{r\tau +t}-\mathbf{1}\bar{x}_{r\tau} \right\| ^2 \right]}
\\
&\quad+\frac{6\left( 1+\rho \right)}{1-\rho}n\mathbb{E} \left[ \left\| \nabla f\left( \bar{x}_{r\tau} \right) \right\| ^2 \right] .
\end{aligned}
\end{equation}
With the help of Lemma~\ref{Lem_client_diver}, and noticing that
\begin{equation}
\begin{aligned}
&\mathbb{E} \left[ \left\| \bar{x}_{\left( r+1 \right) \tau}-\bar{x}_{r\tau} \right\| ^2 \right] 
\\
&=\gamma ^2\mathbb{E} \left[ \left\| \sum_{t=0}^{\tau -1}{\frac{\mathbf{1}^{\top}}{n}\nabla G_{r\tau +t}} \right\| ^2 \right] 
\\
&\leqslant \gamma ^2\tau ^2\mathbb{E} \left[ \left\| \frac{1}{\tau}\sum_{t=0}^{\tau -1}{\frac{\mathbf{1}^{\top}}{n}\nabla F_{r\tau +t}} \right\| ^2 \right] +\gamma ^2\frac{\tau \sigma ^2}{n}
\\
&\leqslant 2\gamma ^2\tau ^2\frac{L^2}{n\tau}\sum_{t=0}^{\tau -1}{\mathbb{E} \left[ \left\| X_{r\tau +t}-\mathbf{1}\bar{x}_{r\tau} \right\| ^2 \right]}
\\
&\quad+2\gamma ^2\tau ^2\mathbb{E} \left[ \nabla f\left( \bar{x}_{r\tau} \right) ^2 \right] +\gamma ^2\frac{\tau \sigma ^2}{n},
\end{aligned}
\end{equation}
we obtain the result with stepsize $\gamma \leqslant \frac{1-\rho}{18\tau L}$.
\end{proof}

\subsection{Proof of Theorem~\ref{Thm_sc}}\label{Subsec_proof_Thm_sc}

With these supporting lemmas in hand, we can prove Theorem~\ref{Thm_sc} under Assumptions~\ref{Ass_graph}-\ref{Ass_sampling}, and the stepsize condition.

To this end, we first prove the contraction of the optimality gap in the following lemma for the strongly convex and smooth objective function.

\begin{lem}[Optimality gap]\label{Lem_opt_gap}
Suppose Assumptions~\ref{Ass_graph}-\ref{Ass_sampling} hold. Let the stepsize satisfy $\gamma \leqslant \frac{1}{11\tau L}$. We have
\begin{equation}
\begin{aligned}
&\mathbb{E} \left[ \left\| \bar{x}_{\left( r+1 \right) \tau}-x^* \right\| ^2 \right] 
\\
&\leqslant \left( 1-\frac{\gamma \mu \tau}{2} \right) \mathbb{E} \left[ \left\| \bar{x}_{r\tau}-x^* \right\| ^2 \right] 
\\
&\quad+9\gamma \tau L\frac{1}{n}\mathbb{E} \left[ \left\| X_{r\tau}-\mathbf{1}\bar{x}_{r\tau} \right\| ^2 \right] +14\gamma ^3\tau ^3L\frac{1}{n}\mathbb{E} \left[ \left\| \varUpsilon _{r\tau} \right\| ^2 \right] 
\\
&\quad+\gamma ^2\frac{\tau \sigma ^2}{n}+5\gamma ^3\tau ^2L\sigma ^2-\gamma \tau \mathbb{E} \left[ f\left( \bar{x}_{rK} \right) -\nabla f\left( x^* \right) \right] .
\end{aligned}
\end{equation}
\end{lem}

\begin{proof}
By the update rules, we have
\begin{equation}
\begin{aligned}
&\mathbb{E} \left[ \left\| \bar{x}_{\left( r+1 \right) \tau}-x^* \right\| ^2 \right] 
\\
&=\mathbb{E} \left[ \left\| \bar{x}_{r\tau}-x^* \right\| ^2 \right] +\gamma ^2\mathbb{E} \left[ \left\| \sum_{t=0}^{\tau-1}{\frac{\mathbf{1}^{\top}}{n}G_{r\tau+t}} \right\| ^2 \right] 
\\
&\quad-2\gamma \mathbb{E} \left[ \left< \bar{x}_{r\tau}-x^*,\sum_{t=0}^{\tau-1}{\frac{\mathbf{1}^{\top}}{n}\nabla F_{r\tau+t}} \right> \right] .
\end{aligned}
\end{equation}
For the second term on the RHS, with the help of \hy{Assumption~\ref{Ass_sampling}}, we have
\begin{equation}
\begin{aligned}
&\mathbb{E} \left[ \left\| \sum_{t=0}^{K-1}{\frac{\mathbf{1}^{\top}}{n}G_{rK+t}} \right\| ^2 \right] 
\\
&=\mathbb{E} \left[ \left\| \sum_{t=0}^{K-1}{\frac{\mathbf{1}^{\top}}{n}\left( G_{rK+t}-\nabla F_{rK+t}+\nabla F_{rK+t} \right)} \right\| ^2 \right] 
\\
&\leqslant 2K\sum_{t=0}^{K-1}{\mathbb{E} \left[ \left\| \frac{\mathbf{1}^{\top}}{n}\nabla F_{rK+t}-\nabla f\left( \bar{x}_{rK} \right) \right\| ^2 \right]}
\\
&\quad+\frac{K\sigma ^2}{n}+2K^2\mathbb{E} \left[ \left\| \nabla f\left( \bar{x}_{rK} \right) \right\| ^2 \right] .
\end{aligned}
\end{equation}
For the third term, we have
\begin{equation}
\begin{aligned}
&\mathbb{E} \left[ \left< \bar{x}_{r\tau}-x^*,\sum_{t=0}^{\tau -1}{\frac{\mathbf{1}^{\top}}{n}\nabla F_{r\tau +t}} \right> \right] 
\\
&=\sum_{t=0}^{\tau -1}{\frac{1}{n}\sum_{i=1}^n{\mathbb{E} \left[ \left< \bar{x}_{r\tau}-x^*,\nabla f_i\left( x_{i,r\tau +t} \right) \right> \right]}}
\\
&=\sum_{t=0}^{\tau -1}{\frac{1}{n}\sum_{i=1}^n{\mathbb{E} \left[ \left< x_{i,r\tau +t}-x^*,\nabla f_i\left( x_{i,r\tau +t} \right) \right> \right]}}
\\
&\quad-\sum_{t=0}^{\tau -1}{\frac{1}{n}\sum_{i=1}^n{\mathbb{E} \left[ \left< x_{i,r\tau +t}-\bar{x}_{r\tau},\nabla f_i\left( x_{i,r\tau +t} \right) \right> \right]}}
\\
&\geqslant \tau \mathbb{E} \left[ f\left( \bar{x}_{r\tau} \right) -\nabla f\left( x^* \right) \right] 
\\
&\quad+\frac{\mu}{2}\sum_{t=0}^{\tau -1}{\frac{1}{n}\sum_{i=1}^n{\mathbb{E} \left[ \left\| x_{i,r\tau +t}-x^* \right\| \right]}}
\\
&\quad-\frac{L}{2}\sum_{t=0}^{\tau -1}{\frac{1}{n}\sum_{i=1}^n{\mathbb{E} \left[ \left\| x_{i,r\tau +t}-\bar{x}_{r\tau} \right\| \right]},}
\end{aligned}
\end{equation}
wherein the last inequality we used the convexity and smoothness of $f_i$ assumed in Assumption~\ref{Ass_cov}.
Then, we get
\begin{equation}
\begin{aligned}
&\mathbb{E} \left[ \left\| \bar{x}_{\left( r+1 \right) \tau}-x^* \right\| ^2 \right] 
\\
&\leqslant \mathbb{E} \left[ \left\| \bar{x}_{r\tau}-x^* \right\| ^2 \right] +\gamma ^2\frac{\tau \sigma ^2}{n}
\\
&\quad+2\gamma ^2\tau L^2\frac{1}{n}\sum_{t=0}^{\tau -1}{\mathbb{E} \left[ \left\| X_{r\tau +t}-\mathbf{1}\bar{x}_{r\tau} \right\| ^2 \right]}
\\
&\quad-2\gamma \tau \left( 1-2\gamma \tau L \right) \mathbb{E} \left[ f\left( \bar{x}_{r\tau} \right) -\nabla f\left( x^* \right) \right] 
\\
&\quad-\gamma \mu \sum_{t=0}^{\tau -1}{\frac{1}{n}\sum_{i=1}^n{\mathbb{E} \left[ \left\| x_{i,r\tau +t}-x^* \right\| \right]}}
\\
&\quad+\gamma L\sum_{t=0}^{K-1}{\frac{1}{n}\sum_{i=1}^n{\mathbb{E} \left[ \left\| x_{i,r\tau +t}-\bar{x}_{rK} \right\| \right]}}.
\end{aligned}
\end{equation}
Noticing that 
\begin{equation}
\begin{aligned}
&\quad-\frac{1}{n}\sum_{i=1}^n{\mathbb{E} \left[ \left\| x_{i,r\tau+t}-x^* \right\| \right]}
\\
&=-\frac{1}{n}\sum_{i=1}^n{\mathbb{E} \left[ \left\| x_{i,r\tau+t}-\bar{x}_{r\tau} \right\| \right]}-\frac{1}{n}\sum_{i=1}^n{\mathbb{E} \left[ \left\| \bar{x}_{r\tau}-x^* \right\| \right]}
\\
&\quad-2\frac{1}{n}\sum_{i=1}^n{\mathbb{E} \left< x_{i,r\tau+t}-\bar{x}_{r\tau}, \bar{x}_{r\tau}-x^* \right>}
\\
&\leqslant -\frac{1}{2n}\sum_{i=1}^n{\mathbb{E} \left[ \left\| \bar{x}_{r\tau}-x^* \right\| \right]}+\frac{1}{n}\sum_{i=1}^n{\mathbb{E} \left[ \left\| x_{i,r\tau+t}-\bar{x}_{r\tau} \right\| \right]}.
\end{aligned}
\end{equation}
We get
\begin{equation}
\begin{aligned}
&\mathbb{E} \left[ \left\| \bar{x}_{\left( r+1 \right) \tau}-x^* \right\| ^2 \right] 
\\
&\leqslant \left( 1-\frac{\gamma \mu \tau}{2} \right) \mathbb{E} \left[ \left\| \bar{x}_{r\tau}-x^* \right\| ^2 \right] +\gamma ^2\frac{\tau \sigma ^2}{n}
\\
&\quad-2\gamma \tau \left( 1-2\gamma \tau L \right) \mathbb{E} \left[ f\left( \bar{x}_{r\tau} \right) -\nabla f\left( x^* \right) \right] 
\\
&\quad+\gamma \left( L+\mu +2\gamma \tau L^2 \right) \frac{1}{n}\sum_{t=0}^{\tau -1}{\mathbb{E} \left[ \left\| X_{r\tau +t}-\mathbf{1}\bar{x}_{r\tau} \right\| ^2 \right]}.
\end{aligned}
\end{equation}
Then, using Lemma~\ref{Lem_client_diver} and letting the stepsize satisfy $\gamma \leqslant \frac{1}{11\tau L}$, we complete the proof.
\end{proof}

Recalling the design of the Lyapunov function in \eqref{Eq_Lyapunov}, we can obtain
\begin{equation}
\begin{aligned}
&\mathbb{E} \left[ \left\| V_{\left( r+1 \right) \tau} \right\| ^2 \right] 
\\
&\leqslant \left( 1-\min \left\{ \frac{\gamma \mu}{2},\frac{1-\rho}{8} \right\} \right) \mathbb{E} \left[ \left\| V_{r\tau} \right\| ^2 \right] 
\\
&\quad +e_1\mathbb{E} \left[ \left\| X_{r\tau}-1\bar{x}_{r\tau} \right\| ^2 \right] +e_2\mathbb{E} \left[ \left\| \varUpsilon _{r\tau} \right\| ^2 \right] 
\\
&\quad +e_3\mathbb{E} \left[ f\left( \bar{x}_{r\tau} \right) -f\left( x^* \right) \right] +\gamma ^2\tau \frac{\sigma ^2}{n}+\frac{9}{2}\tau ^2\gamma ^3L\sigma ^2
\\
&\quad +\frac{240\gamma ^3\tau ^2L\rho}{1-\rho}\sigma ^2+\frac{24892\gamma ^3\tau ^2L}{\left( 1-\rho \right) ^3},
\end{aligned}
\end{equation}
where
\begin{equation}
\begin{aligned}
e_1&:=9\gamma \tau L\frac{1}{n}+\frac{54L^2}{1-\rho}c_y-\frac{1-\rho}{8}c_x,
\\
e_2&:=14\gamma ^3\tau ^3L\frac{1}{n}+\frac{5\rho}{1-\rho}\gamma ^2\tau ^2c_x-\frac{1-\rho}{8}c_y,
\\
e_3&:=-\gamma \tau +\frac{18\rho}{1-\rho}\gamma ^2\tau ^2Lnc_x+\frac{36}{1-\rho}nLc_y.
\end{aligned}
\end{equation}
Noticing the coefficients of the Lyapunov function are designed as \eqref{Eq_Lya_coff} and letting the stepsize satisfy $\gamma =\mathcal{O} \left( \frac{\left( 1-\rho \right) ^2}{L\tau} \right)$, we have $e_1, e_2, e_3 <0$, we obtain the convergence rate in \eqref{Eq_thm_sc_conv}.

To further get the communication complexity, we tune the stepsize between
\begin{equation}
\left\{ \gamma , \frac{2\ln \left( \max \left\{ 2,\tau ^2\mu ^2R/H_1,\tau ^3\mu ^3R/H_2 \right\} \right)}{\tau \mu R} \right\}
\end{equation}
to ensure that gradient noise-related errors
match the linear part, where
\begin{equation*}
H_1=\frac{\tau \sigma ^2}{n\mathbb{E} \left[ \left\| V_0 \right\| ^2 \right]}, H_2=\frac{\tau ^2L\sigma ^2}{\left( 1-\rho \right) ^3\mathbb{E} \left[ \left\| V_0 \right\| ^2 \right]}.
\end{equation*}
Then, we get the required number of communication rounds to achieve an accuracy of $\epsilon$ in \eqref{Eq_thm_sc_comm}.

\subsection{Proof of Theorem~\ref{Thm_nc}}\label{Subsec_proof_Thm_nc}
For the nonconvex case, we first derive the following descent lemma using the smoothness of $f_i$.

\begin{lem}
Suppose Assumptions~\ref{Ass_graph},~\ref{Ass_smoothness} and~\ref{Ass_sampling} hold. Let the stepsize satisfy $\gamma \leqslant \frac{1}{2\tau L}$.
We get
\begin{equation}
\begin{aligned}
&\mathbb{E} \left[ f\left( \bar{x}_{\left( r+1 \right) \tau} \right) \right] 
\\
&\leqslant \mathbb{E} \left[ f\left( \bar{x}_{r\tau} \right) \right] -\frac{\gamma \tau}{4}\mathbb{E} \left[ \left\| \nabla f\left( \bar{x}_{r\tau} \right) \right\| ^2 \right] 
\\
&\quad+3\gamma L^2\frac{1}{n}\sum_{t=0}^{\tau -1}{\mathbb{E} \left[ \left\| X_{r\tau +t}-\mathbf{1}\bar{x}_{r\tau} \right\| ^2 \right]}+\frac{\gamma ^2\tau L\sigma ^2}{2n}.
\end{aligned}
\end{equation}
\end{lem}

\begin{proof}
By the smoothness of $f_i$ assumed in Assumption~\ref{Ass_smoothness}, we have
\begin{equation}\label{Eq_descent_lem_1}
\begin{aligned}
&f\left( \bar{x}_{\left( r+1 \right) \tau} \right) 
\\
&\leqslant f\left( \bar{x}_{r\tau} \right) +\left< \nabla f\left( \bar{x}_{r\tau} \right) ,\bar{x}_{\left( r+1 \right) \tau}-\bar{x}_{r\tau} \right> 
\\
&\quad+\frac{L}{2}\left\| \bar{x}_{\left( r+1 \right) \tau}-\bar{x}_{r\tau} \right\| ^2.
\end{aligned}
\end{equation}
Then, for the last term on the RHS of \eqref{Eq_descent_lem_1}, we have
\begin{equation}\label{Eq_des_lem_norm_term}
\begin{aligned}
&\mathbb{E} \left[ \left\| \bar{x}_{\left( r+1 \right) \tau}-\bar{x}_{r\tau} \right\| ^2 \right] 
\\
&\leqslant \gamma ^2\tau ^2\mathbb{E} \left[ \left\| \frac{1}{\tau}\sum_{t=0}^{\tau -1}{\frac{\mathbf{1}^{\top}}{n}\nabla F_{r\tau +t}} \right\| ^2 \right] +\gamma ^2\frac{\tau \sigma ^2}{n}
\\
&\leqslant 2\gamma ^2\tau ^2\frac{L^2}{n\tau}\sum_{t=0}^{\tau -1}{\mathbb{E} \left[ \left\| X_{r\tau +t}-\mathbf{1}\bar{x}_{r\tau} \right\| ^2 \right]}
\\
&\quad+2\gamma ^2\tau ^2\mathbb{E} \left[ \nabla f\left( \bar{x}_{r\tau} \right) ^2 \right] +\gamma ^2\frac{\tau \sigma ^2}{n}.
\end{aligned}
\end{equation}

For the inner-product term above, we have
\begin{equation}\label{Eq_des_lem_inner_p}
\begin{aligned}
&\mathbb{E} \left[ \left< \nabla f\left( \bar{x}_{r\tau} \right) ,\bar{x}_{\left( r+1 \right) \tau}-\bar{x}_{r\tau} \right> \right] 
\\
&=\mathbb{E} \left[ \left< \nabla f\left( \bar{x}_{r\tau} \right) ,-\gamma \sum_{t=0}^{\tau -1}{\frac{\mathbf{1}^{\top}}{n}\nabla F_{r\tau +t}} \right> \right] 
\\
&=-\gamma \tau \mathbb{E} \left[ \left\| \nabla f\left( \bar{x}_{r\tau} \right) \right\| ^2 \right] 
\\
&\quad+\mathbb{E} \left[ \left< \nabla f\left( \bar{x}_{r\tau} \right) ,-\gamma \sum_{t=0}^{\tau -1}{\left( \frac{\mathbf{1}^{\top}}{n}\nabla F_{r\tau +t}-\nabla f\left( \bar{x}_{r\tau} \right) \right)} \right> \right] 
\\
&\leqslant -\frac{\gamma \tau}{2}\mathbb{E} \left[ \left\| \nabla f\left( \bar{x}_{r\tau} \right) \right\| ^2 \right] 
\\
&\quad+2\gamma L^2\frac{1}{n}\sum_{t=0}^{\tau -1}{\mathbb{E} \left[ \left\| X_{r\tau +t}-\mathbf{1}\bar{x}_{r\tau} \right\| ^2 \right]}.
\end{aligned}
\end{equation}

Applying \eqref{Eq_des_lem_inner_p} and \eqref{Eq_des_lem_norm_term} into \eqref{Eq_descent_lem_1} and letting the stepsize satisfy $\gamma \leqslant \frac{1}{\tau L}$, we complete the proof.
\end{proof}

Then, using Lemma~\ref{Lem_client_diver} to bound the client divergence during the local updates and doing accumulation, we have 
\begin{equation}\label{Eq_des_lem_accu}
\begin{aligned}
&\frac{1}{R}\sum_{r=0}^{R-1}{\mathbb{E} \left[ \left\| \nabla f\left( \bar{x}_{r\tau} \right) \right\| ^2 \right]}
\\
&\leqslant \frac{4\left( f\left( \bar{x}_0 \right) -f\left( \bar{x}_{r\tau} \right) \right)}{\gamma \tau R}+2\gamma L\frac{\sigma ^2}{n}+18\gamma ^2\tau L^2\sigma ^2
\\
&\quad+36L^2\frac{1}{nR}\sum_{r=0}^{R-1}{\mathbb{E} \left[ \left\| X_{r\tau}-\mathbf{1}\bar{x}_{r\tau} \right\| ^2 \right]}
\\
&\quad+54\gamma ^2\tau ^2L^2\frac{1}{nR}\sum_{r=0}^{R-1}{\mathbb{E} \left[ \left\| \varUpsilon _{r\tau} \right\| ^2 \right]}.
\end{aligned}
\end{equation}

For the accumulated consensus error, invoking Lemma~\ref{Lem_cons_err}, we have 
\begin{equation}
\begin{aligned}
&\frac{1}{R}\sum_{r=1}^R{\mathbb{E} \left[ \left\| X_{r\tau}-\mathbf{1}\bar{x}_{r\tau} \right\| ^2 \right]}
\\
&\leqslant \frac{3+\rho}{4}\frac{1}{R}\sum_{r=0}^{R-1}{\mathbb{E} \left[ \left\| X_{r\tau}-\mathbf{1}\bar{x}_{r\tau} \right\| ^2 \right]}
\\
&\quad+\frac{5\rho \gamma ^2\tau ^2}{1-\rho}\frac{1}{R}\sum_{r=0}^{R-1}{\mathbb{E} \left[ \left\| \varUpsilon _{r\tau} \right\| ^2 \right]}+3\rho \gamma ^2\tau n\sigma ^2
\\
&\quad+\frac{9\rho \gamma ^2\tau ^2n}{1-\rho}\frac{1}{R}\sum_{r=0}^{R-1}{\mathbb{E} \left[ \left\| \nabla f\left( \bar{x}_{r\tau} \right) \right\| ^2 \right]}.
\end{aligned}
\end{equation}
Adding $\left\| X_0-\mathbf{1}\bar{x}_0 \right\| ^2$ on the both side and noticing that
\[
\frac{1}{R}\sum_{r=0}^{R-1}{\mathbb{E} \left[ \left\| X_{r\tau}-\mathbf{1}\bar{x}_{r\tau} \right\| ^2 \right]}\leqslant \frac{1}{R}\sum_{r=0}^R{\mathbb{E} \left[ \left\| X_{r\tau}-\mathbf{1}\bar{x}_{r\tau} \right\| ^2 \right]},
\]
we obtain
\begin{equation}\label{Eq_accu_cons_error_1}
\begin{aligned}
&\frac{1}{R}\sum_{r=0}^{R-1}{\mathbb{E} \left[ \left\| X_{r\tau}-\mathbf{1}\bar{x}_{r\tau} \right\| ^2 \right]}
\\
&\leqslant \frac{4\mathbb{E} \left[ \left\| X_0-\mathbf{1}\bar{x}_0 \right\| ^2 \right]}{\left( 1-\rho \right) R}+\frac{20\rho \gamma ^2\tau ^2}{\left( 1-\rho \right) ^2}\frac{1}{R}\sum_{r=0}^{R-1}{\mathbb{E} \left[ \left\| \varUpsilon _{r\tau} \right\| ^2 \right]}
\\
&\quad+\frac{36\gamma ^2\tau ^2\rho}{\left( 1-\rho \right) ^2}n\frac{1}{R}\sum_{r=0}^{R-1}{\mathbb{E} \left[ \left\| \nabla f\left( \bar{x}_{r\tau} \right) \right\| ^2 \right]}+\frac{12\gamma ^2\tau \rho}{1-\rho}n\sigma ^2.
\end{aligned}
\end{equation}

Similarly, for the accumulated gradient tracking error, we have
\begin{equation}\label{Eq_accu_tracking_error_1}
\begin{aligned}
&\frac{1}{R}\sum_{r=0}^{R-1}{\mathbb{E} \left[ \left\| \varUpsilon _{r\tau} \right\| ^2 \right]}
\\
&\leqslant \frac{4\mathbb{E} \left[ \left\| \varUpsilon _0 \right\| ^2 \right]}{\left( 1-\rho \right) R}+\frac{28n\sigma ^2}{\tau \left( 1-\rho \right)}
\\
&\quad+\frac{72}{\left( 1-\rho \right) ^2}n\frac{1}{R}\sum_{r=0}^{R-1}{\mathbb{E} \left[ \left\| \nabla f\left( \bar{x}_{r\tau} \right) \right\| ^2 \right]}
\\
&\quad+\frac{216L^2}{\left( 1-\rho \right) ^2}\frac{1}{R}\sum_{r=0}^{R-1}{\mathbb{E} \left[ \left\| X_{r\tau}-\mathbf{1}\bar{x}_{r\tau} \right\| ^2 \right]}.
\end{aligned}
\end{equation}

Next, we decouple the accumulated consensus error and gradient tracking error in the following two lemmas.

\begin{lem}\label{Lem_cons_err_accu}
Suppose Assumptions~\ref{Ass_graph},~\ref{Ass_smoothness} and~\ref{Ass_sampling} hold. Let the stepsize satisfy $\gamma \leqslant \min \left\{ \frac{1-\rho}{12\tau L}\,,\,\frac{\left( 1-\rho \right) ^2}{93\tau L\sqrt{\rho}}\, \right\} $.
We get
\begin{equation}
\begin{aligned}
&\frac{1}{R}\sum_{r=0}^{R-1}{\mathbb{E} \left[ \left\| X_{r\tau}-\mathbf{1}\bar{x}_{r\tau} \right\| ^2 \right]}
\\
&\leqslant C_{0,x}+\frac{24\gamma ^2\tau \rho}{1-\rho}n\sigma ^2+\frac{1120\rho \gamma ^2\tau}{\left( 1-\rho \right) ^3}n\sigma ^2
\\
&\quad+\frac{2952\rho \gamma ^2\tau ^2n}{\left( 1-\rho \right) ^4}\frac{1}{R}\sum_{r=0}^{R-1}{\mathbb{E} \left[ \left\| \nabla f\left( \bar{x}_{r\tau} \right) \right\| ^2 \right]},
\end{aligned}
\end{equation}
where 
\begin{equation}
C_{0,x}:=\frac{8\mathbb{E} \left[ \left\| X_0-\mathbf{1}\bar{x}_0 \right\| ^2 \right]}{\left( 1-\rho \right) R}+\frac{160\rho \gamma ^2\tau ^2}{\left( 1-\rho \right) ^3}\frac{\mathbb{E} \left[ \left\| \varUpsilon _0 \right\| ^2 \right]}{R}.
\end{equation}
\end{lem}

\begin{proof}
Applying \eqref{Eq_accu_tracking_error_1} to \eqref{Eq_accu_cons_error_1}, 
we complete the proof by combining like terms and letting the stepsize satisfy the condition. 
\end{proof}

\begin{lem}\label{Lem_tracking_err_accu}
Suppose Assumptions~\ref{Ass_graph},~\ref{Ass_smoothness} and~\ref{Ass_sampling} hold. Let the stepsize $\gamma \leqslant \min \left\{\frac{1-\rho}{12\tau L}\,,\,\frac{\left( 1-\rho \right) ^2}{62\tau L\sqrt{\rho}}\, \right\} $. Then, we have
\begin{equation}
\begin{aligned}
&\frac{1}{R}\sum_{r=0}^{R-1}{\mathbb{E} \left[ \left\| \varUpsilon _{r\tau} \right\| ^2 \right]}
\\
&\leqslant C_{0,y}+\frac{54n\sigma ^2}{\tau \left( 1-\rho \right)}+\frac{2304\gamma ^2\tau L^2\rho}{\left( 1-\rho \right) ^3}n\sigma ^2
\\
&\quad+\frac{146}{\left( 1-\rho \right) ^2}\frac{n}{R}\sum_{r=0}^{R-1}{\mathbb{E} \left[ \left\| \nabla f\left( \bar{x}_{r\tau} \right) \right\| ^2 \right]},
\end{aligned}
\end{equation}
where 
\begin{equation}
C_{0,y}:=\frac{8\mathbb{E} \left[ \left\| \varUpsilon_0 \right\| ^2 \right]}{\left( 1-\rho \right) R}+\frac{768L^2}{\left( 1-\rho \right) ^3}\frac{\mathbb{E} \left[ \left\| X_0-\mathbf{1}\bar{x}_0 \right\| ^2 \right]}{R}.
\end{equation}
\begin{proof}
Applying \eqref{Eq_accu_cons_error_1} into \eqref{Eq_accu_tracking_error_1}, 
we complete the proof by combining like terms and letting the stepsize satisfy the condition.
\end{proof}
\end{lem}

Applying Lemmas~\ref{Lem_cons_err_accu} and \ref{Lem_tracking_err_accu} to \eqref{Eq_des_lem_accu} and letting the stepsize satisfy
\begin{equation*}
\gamma \leqslant \min \left\{ \frac{1-\rho}{178\tau L}, \frac{\left( 1-\rho \right) ^2}{625\sqrt{\rho}\tau L} \right\}, 
\end{equation*}
we obtain the sublinear convergence rate \eqref{Eq_thm_nc_conv} in Theorem~\ref{Thm_nc}.

To further derive the communication complexity, we tune the stepsize between
\begin{equation}
\left\{ \gamma ,\frac{1}{\sqrt{H_1R}},\frac{1}{\sqrt[3]{H_2R}} \right\},
\end{equation}
where 
\begin{equation*}
\begin{aligned}
H_1&=\frac{\tau \sigma ^2L}{n\left( f\left( \bar{x}_0 \right) -f\left( \bar{x}_{R\tau} \right) \right)},
\\
H_2&=\frac{\tau ^3L^2}{\left( 1-\rho \right) ^3\left( f\left( \bar{x}_0 \right) -f\left( \bar{x}_{R\tau} \right) \right)}\frac{\sigma ^2}{\tau},
\end{aligned}
\end{equation*}
and we obtain the required communication rounds to achieve an accuracy of $\epsilon>0$ in \eqref{Eq_thm_nc_comm}.

\end{document}